\documentclass{article}

\PassOptionsToPackage{numbers, sort&compress}{natbib}
\usepackage[main, final]{neurips_2026}

\usepackage[utf8]{inputenc}
\usepackage[hidelinks,hyperfootnotes=false]{hyperref}
\usepackage{url}

\usepackage{booktabs}
\usepackage{colortbl}

\usepackage{amsthm}
\usepackage{amsmath}
\usepackage{mathtools}
\usepackage{bm}

\usepackage{microtype}
\usepackage{xcolor}
\definecolor{oursblue}{HTML}{EAF4FF}
\definecolor{highlightpink}{HTML}{E34A92}
\definecolor{boxblue}{HTML}{5375F1}

\usepackage{enumitem}
\usepackage[most]{tcolorbox}
\usepackage{amssymb}

\usepackage{graphicx}
\usepackage{float}
\usepackage{afterpage}
\usepackage{tikz}
\usetikzlibrary{calc}

\usepackage{multirow}
\usepackage{tabularx}

\usepackage[T1]{fontenc}
\usepackage[section]{placeins}

\newcommand{\E}{\mathbb{E}}
\newcommand{\R}{\mathbb{R}}
\newcommand{\bX}{\boldsymbol{X}}
\newcommand{\bZ}{\boldsymbol{Z}}
\newcommand{\bx}{\boldsymbol{x}}
\newcommand{\bz}{\boldsymbol{z}}
\newcommand{\mI}{\bm{I}}
\newcommand{\rf}{\mathtt{RF}}

\theoremstyle{plain}
\newtheorem{theorem}{Theorem}[section]
\newtheorem{proposition}[theorem]{Proposition}

\theoremstyle{definition}

\newtheorem{assumption}[theorem]{Assumption}
\theoremstyle{remark}
\newtheorem{remark}[theorem]{Remark}

\AtBeginEnvironment{tcolorbox}{\normalcolor}

\makeatletter
\if@preprint
  \renewcommand{\@notice}{}
\fi
\makeatother

\tcbset{
  colback=boxblue!7,
  colframe=boxblue!70!black,
  boxrule=0.6pt,
  arc=2pt,
}

\lstdefinestyle{paperpython}{
  language=Python,
  basicstyle=\ttfamily\fontsize{7.6}{9.0}\selectfont,
  keywordstyle=\color{boxblue!85!black}\bfseries,
  commentstyle=\color{black!48}\itshape,
  stringstyle=\color{highlightpink!80!black},
  identifierstyle=\color{black!88},
  numbers=left,
  numberstyle=\scriptsize\color{black!35},
  numbersep=5pt,
  xleftmargin=1.65em,
  showstringspaces=false,
  columns=fullflexible,
  keepspaces=true,
  breaklines=true,
  breakatwhitespace=true,
  tabsize=4,
  aboveskip=0pt,
  belowskip=0pt,
}

\newtcblisting{paperpythonlisting}[2][]{%
  enhanced jigsaw,
  breakable,
  listing only,
  listing options={style=paperpython},
  colback=black!1.2,
  colframe=black!13,
  coltitle=black!82,
  boxrule=0.45pt,
  arc=1.2pt,
  outer arc=1.2pt,
  left=0.8mm,
  right=1.0mm,
  top=1.0mm,
  bottom=1.0mm,
  before skip=0.55\baselineskip,
  after skip=0.55\baselineskip,
  borderline west={0.9pt}{0pt}{boxblue!80!black},
  fonttitle=\sffamily\footnotesize\bfseries,
  attach boxed title to top left={xshift=2mm,yshift=-1.2mm},
  boxed title style={
    enhanced jigsaw,
    colback=white,
    colframe=black!13,
    boxrule=0.45pt,
    arc=1pt,
    outer arc=1pt,
    left=1.2mm,
    right=1.2mm,
    top=0.45mm,
    bottom=0.45mm,
  },
  title={#2},
  #1
}

\newcommand{\sgn}{\mathtt{sign}}
\newcommand{\sgnrf}{\mathtt{signRF}}
\title{Signed Rectified Flow:\\ Negativity-Controlled Generation}

\author{%
  Runlong Liao\thanks{Co-first authors.}%
  \quad
  Baiyu Su\footnotemark[\value{footnote}]%
  \quad
  Lizhang Chen %
  \quad
  Qiang Liu %
  \\[6pt]
  The University of Texas at Austin\\[0pt]
  {\small\textcolor[HTML]{2E4FAF}{\url{https://rectifiedflow.github.io/Signed-RF/}}}%
}

\begin{document}

\newskip\origtextfloatsep \origtextfloatsep=\textfloatsep
\setlength{\textfloatsep}{8pt plus 2pt minus 2pt}
\afterpage{\global\textfloatsep=\origtextfloatsep}
\maketitle

\vspace{-2.3em}
\begin{abstract}
    We introduce \emph{Signed Rectified Flow (Signed RF)}, a generalization of Rectified Flow that targets a signed measure
    $\pi^{\mathtt{sign}} = (1+\alpha)\,\pi^+ - \alpha \,\pi^-$,
    where $\alpha>0$, $\pi^+$ represents the distribution to promote, and $\pi^-$ represents the distribution to suppress.
    Although sampling from a signed measure is not well-defined, Signed RF induces a valid generative process that concentrates on the positive region of $\pi^{\mathtt{sign}}$ while provably excluding regions dominated by the negative component. This yields a principled framework for incorporating negative information and exclusion constraints into generative modeling. Theoretically, we analyze the signed continuity equation underlying Signed RF
    and explain how negative mass creates exclusion barriers through a charged-particle interpretation. Empirically, Signed RF leads to practical adaptive guidance algorithms. Across applications, Signed RF improves the fidelity--diversity trade-off on ImageNet, reduces nearest-neighbor similarity in anti-memorization tests, and reduces adversarial-prompt nudity in SD 3.5 while preserving CLIP and aesthetic scores.

\end{abstract}

\vspace{-1.1em}

\begin{figure}[H]
    \vspace{-0.5em}
    \centering
    \makebox[\textwidth][c]{%
        \begin{minipage}{1.0\textwidth}
            \centering
            \includegraphics[width=\textwidth]{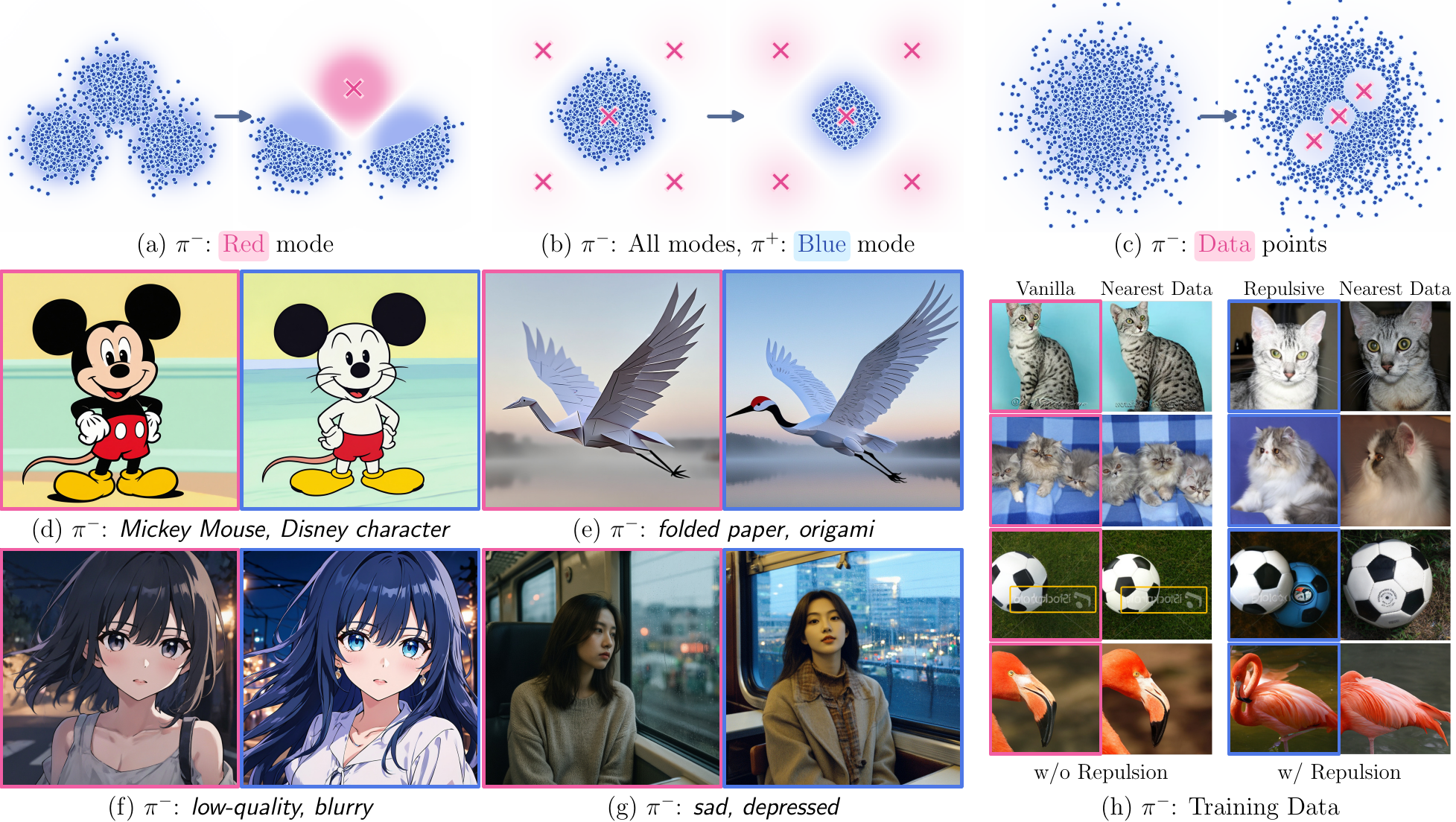}
            \vspace{-1.3em}
        \end{minipage}%
    }

    \caption{Signed Rectified Flow steers samples toward the \textcolor[HTML]{2E4FAF}{positive target \(\pi^+\)} while avoiding regions where the \textcolor[HTML]{E34A92}{negative target \(\pi^-\)} dominates.
        Top row: 2D toy examples. (a) suppresses an undesired mode; (b) strengthens the desired target mode; and (c) avoids generating protected data.
        Bottom row: image-generation examples.
        (d) preserves cartoon attributes in \(\pi^+\) (round black ears, red shorts, yellow shoes) while avoiding the copyrighted character identity in \(\pi^-\);
        (e) preserves a paper-crane scene in \(\pi^+\) while avoiding a folded-paper appearance in \(\pi^-\);
        (f)--(g) suppress undesired visual attributes in \(\pi^-\);
        and (h) mitigates data leakage by repelling samples from the training set.}

    \label{fig:signed_rf_teaser}
\end{figure}

\section{Introduction}
Many generative tasks can be viewed as \emph{promoting} a desirable distribution \(\pi^+\) while \emph{suppressing} an undesirable one \(\pi^-\). In conditional generation, this corresponds to producing samples that better match the target condition and avoiding inconsistent outputs~\cite{wei2021finetuned, sanh2021multitask, ho2022classifier, dhariwal2021diffusion, rombach2022high, nichol2021glide, saharia2022photorealistic,liao2026momentum}. In safety and alignment settings, the objective is to generate value-aligned outputs while excluding harmful ones~\cite{bai2022training, bai2022constitutional, schramowski2023safe}. In preference learning, given preferred and dispreferred samples, the goal is to steer generation toward the preferred distribution~\cite{ouyang2022training, rafailov2023direct, black2023training, wallace2024diffusion, lee2023aligning}. Across these settings, the central challenge is to design mechanisms that reliably amplify desirable modes while discouraging undesirable ones.

For flow and diffusion models~\cite{ho2020denoising, song2020score, song2020denoising, liu2022flow, liu2022rectified, lipman2022flow, albergo2022building, su2026rectified}, the standard approach is to apply guidance~\cite{dhariwal2021diffusion, ho2022classifier, nichol2021glide, saharia2022photorealistic, rombach2022high, liao2026momentum} at sampling time, where trajectories are steered by extrapolating between positive and negative fields~\cite{ho2022classifier}.
However, such procedures remain largely heuristic: their induced sampling distributions are poorly understood, making it unclear which regions of the data space are ultimately promoted or suppressed.

We propose \emph{signed measures} as a principled mechanism for incorporating negative information into generative modeling. Given a distribution $\pi^+$ to favor and $\pi^-$ to avoid, we consider the signed target
\[
  \pi^{\sgn}(\bx)
  =
  (1+\alpha)\,\pi^+(\bx)
  -
  \alpha\,\pi^-(\bx),
  \qquad \alpha>0.
\]
The goal is to use negativity to encode more than the mere absence of positive data: the signed target explicitly specifies what the model should avoid. Although \(\pi^{\sgn}\) is not generally a probability distribution and cannot be sampled from in the usual sense, we do not seek to sample from it directly. Instead, we use it as an intermediate object for constructing a valid
generative procedure that follows its positive part while avoiding regions
dominated by negative mass.

To achieve this, we formally extend Rectified Flow~\citep{liu2022flow, liu2022rectified} to signed targets by extrapolating the standard RF velocity formula to \(\pi^{\sgn}\). As illustrated in Fig.~\ref{fig:signed_density_evolution}, simulating the resulting dynamics from the source distribution generates samples from a dynamically reachable subset of the positive region of the signed target. Regions in which \(\pi^{\sgn}\) is negative are never sampled, while an additional unsampled \emph{buffer region} separates the reachable and negative regions. Signed RF therefore rectifies the signed target into a valid sampling law by preserving the signed density on its reachable region and assigning zero density elsewhere.

The paper is organized as follows. In Sec.~\ref{sec:rf_to_signed_rf}, we derive Signed RF by extending Rectified Flow from convex mixtures to signed targets. In Sec.~\ref{sec:signed_rf_sampling}, we characterize the induced sampling law through the notions of reachable and buffer regions, provide a charged-particle interpretation of the dynamics, and establish the main theoretical guarantees, including density preservation on the reachable region and nonpenetration into negative regions. In Sec.~\ref{sec:signed_rf_practice}, we derive a practical guidance-form implementation whose strength is determined by a density ratio, estimated either with a learned classifier or through online ODE tracking. Finally, in Sec.~\ref{sec:experiments}, we demonstrate Signed RF across a range of applications, including conditional generation, safe generation, trajectory planning, concept suppression, and anti-memorization.

\section{Method}
\subsection{From Rectified Flow to Signed Rectified Flow}
\label{sec:rf_to_signed_rf}

\paragraph{Rectified Flow.}
Let \(\pi_1\) be the target data distribution on \(\mathbb{R}^d\), and let
\(\pi_0\) be a noise source, such as \(\mathcal N(0,\mI)\).
Rectified Flow (RF)~\citep{liu2022flow,liu2022rectified,lipman2022flow,albergo2022building}
constructs a linear interpolation \(\bX_t=t\bX_1+(1-t)\bX_0\), where
\(\bX_0\sim\pi_0\) and \(\bX_1\sim\pi_1\) are independent. This interpolation
induces the ODE
\[
    \frac{\mathrm d \bZ_t}{\mathrm d t}=v_t^{\rf}(\bZ_t),
    \qquad \text{starting from } \bZ_0\sim\pi_0,
\]
whose velocity field is given by
\begin{align}
    \label{eq:rf_velocity}
    v_t^{\rf}(\bx)
    =
    \E\bigl[\bX_1 - \bX_0 \mid \bX_t = \bx\bigr].
\end{align}
Let \(\pi_t\) denote the marginal distribution of \(\bX_t\).
A key property of RF is marginal preservation: the ODE solution satisfies \(\bZ_t\sim\pi_t\) for all \(t\in[0,1]\), and in particular
\(\bZ_1\sim\pi_1\). In practice, the velocity field \(v_t^{\rf}\) is approximated by a neural
network \(v^\theta\) trained with
\[
    \mathcal L(\theta)
    =
    \int_0^1\E\Bigl[\|v^\theta(\bX_t,t)-(\bX_1-\bX_0)\|^2\Bigr]\,\mathrm{d}t.
\]
\paragraph{Rectified Flow for Convex Mixtures.}
We next examine how Rectified Flow behaves under mixtures of target
distributions. Let \((\pi_t^+,v_t^+)\) and \((\pi_t^-,v_t^-)\) denote the RF
marginals and velocity fields induced, respectively, by
\(\bX_t^+\!=\!(1-t)\bX_0+t\bX_1^+\) and
\(\bX_t^-\!=\!(1-t)\bX_0+t\bX_1^-\), where
\(\bX_0\sim\pi_0\), \(\bX_1^+\sim\pi_1^+\), and
\(\bX_1^-\sim\pi_1^-\). For the convex mixture
\(\pi_1^{\mathtt{mix}}\!=\!(1-w)\pi_1^+ + w\pi_1^-\), with
\(w\in[0,1]\), linearity of the RF marginal and flux gives
\(\pi_t^{\mathtt{mix}}\!=\!(1-w)\pi_t^+ + w\pi_t^-\).
Applying Eq.~\ref{eq:rf_velocity} to the mixture gives the RF velocity field
\begin{equation}
    v_t^{\mathtt{mix}}(\bx)
    =
    \frac{
        (1-w)\,\pi_t^+(\bx)\,v_t^+(\bx)
        +
        w\,\pi_t^-(\bx)\,v_t^-(\bx)
    }{
        (1-w)\,\pi_t^+(\bx)
        +
        w\,\pi_t^-(\bx)
    }.
    \label{eq:mixture_velocity}
\end{equation}
By the marginal-preservation property of RF, integrating
\(\dot{\bZ}_t\!=\!v_t^{\mathtt{mix}}(\bZ_t)\)
from \(\bZ_0 \!\sim\!\pi_0\) yields
\(\bZ_t\sim\pi_t^{\mathtt{mix}}\) for all \(t\in[0,1]\), and in particular
\(\bZ_1\sim\pi_1^{\mathtt{mix}}\).
For \(w\in[0,1]\), \(\pi_t^{\mathtt{mix}}\) remains a nonnegative,
normalized density for all \(t\), and the associated dynamics therefore define
a standard probability flow. A natural question, however, is what happens when
\(w\) is allowed to be negative.

\paragraph{Rectified Flow for Signed Mixtures.}
We now extend the mixture coefficient beyond the convex regime by setting
\(w\!=\!-\alpha\), where \(\alpha>0\). This gives the signed target
\(\pi_1^{\sgn}(\bx)\!\coloneqq\!(1+\alpha)\pi_1^+(\bx)-\alpha\pi_1^-(\bx)\),
which has unit total mass but need not be nonnegative. By linearity of the RF
marginals, the corresponding signed marginal is
\(\pi_t^{\sgn}(\bx)\!=\!(1+\alpha)\pi_t^+(\bx)-\alpha\pi_t^-(\bx)\) for
\(t\in[0,1]\). Substituting \(w\!=\!-\alpha\) into
Eq.~\ref{eq:mixture_velocity} yields the Signed RF velocity field
\begin{equation}
    v_t^{\sgnrf}(\bx)
    =
    \frac{
        (1+\alpha)\,\pi_t^+(\bx)\,v_t^+(\bx)
        -
        \alpha\,\pi_t^-(\bx)\,v_t^-(\bx)
    }{
        (1+\alpha)\,\pi_t^+(\bx)
        -
        \alpha\,\pi_t^-(\bx)
    }.
    \label{eq:signed_rf_velocity}
\end{equation}
The denominator is precisely \(\pi_t^{\sgn}(\bx)\), and hence
\(v_t^{\sgnrf}\) is well defined away from the zero set
\(\Omega_t^0\!\coloneqq\!\{\bx:\pi_t^{\sgn}(\bx)=0\}\). Although \(\pi_t^{\sgn}\) may fail to be a probability density, its initial
marginal remains the valid source distribution:
\(\pi_0^{\sgn}\!=\!\pi_0\).
We therefore define Signed RF through the source-initialized ODE
\(\dot{\bZ}_t\!=\!v_t^{\sgnrf}(\bZ_t)\), with
\(\bZ_0\sim\pi_0\), and denote the law of \(\bZ_t\) by
\(\pi_t^{\sgnrf}\). As we establish below, almost every source-initialized trajectory remains in the positive
region \(\Omega_t^+\!\coloneqq\!\{\bx:\pi_t^{\sgn}(\bx)>0\}\) and does not
cross the singular zero set \(\Omega_t^0\). The singular boundary therefore
acts as a repulsive barrier for the realized dynamics, allowing the ODE to
remain well defined along sampled trajectories even when the signed marginal
develops negative regions.

\subsection{Sampling Behavior of Signed RF}
\label{sec:signed_rf_sampling}
By construction, \(\pi_t^{\sgnrf}\), as the law of the source-initialized ODE, is a valid probability distribution. In contrast, \(\pi_t^{\sgn}\) has unit total mass but may take negative values. Therefore, once negative regions emerge, \(\pi_t^{\sgnrf}\) can no longer coincide globally with \(\pi_t^{\sgn}\). This raises the central question: which part of the signed marginal is realized by the source-initialized flow? The one-dimensional example in Fig.~\ref{fig:signed_density_evolution} illustrates the resulting structure. The black curve shows the signed marginal \(\pi_t^{\sgn}\), while the blue histogram shows the empirical density obtained by simulating the ODE driven by \(v_t^{\sgnrf}\) from \(\bZ_0\sim\pi_0\). Their evolution reveals three key phenomena.

\begin{figure}[h!]
    \centering
    \includegraphics[width=\linewidth]{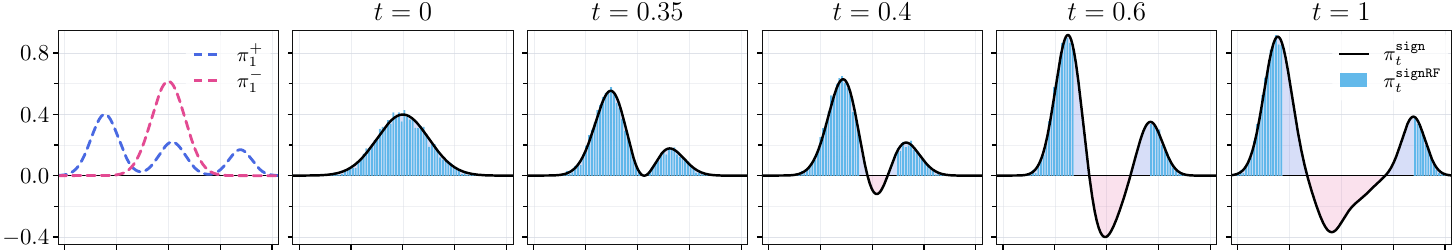}
    \caption{1D evolution of the signed density \(\pi_t^{\sgn}\) and the
    induced sampling law \(\pi_t^{\sgnrf}\). Signed RF rejects negative
    regions and matches the signed density on a reachable subset of the
    positive region.}
    \label{fig:signed_density_evolution}
\end{figure}

\begin{itemize}[leftmargin=1.2em]

    \item \textbf{Trajectories remain in the positive region.}
    Almost every source-initialized trajectory remains in the positive region \(\Omega_t^+\!\coloneqq\!\{\bx:\pi_t^{\sgn}(\bx)>0\}\) and never enters the negative region \(\Omega_t^-\!\coloneqq\!\{\bx:\pi_t^{\sgn}(\bx)<0\}\). However, these trajectories generally occupy only a subset of \(\Omega_t^+\). We define \(\Omega_t^r\!\coloneqq\!\operatorname{supp}(\pi_t^{\sgnrf})\) as the \emph{reachable region}, shown in blue in Fig.~\ref{fig:signed_density_evolution}.

    \item \textbf{Signed RF rectifies the signed marginal.}
    On the reachable region \(\Omega_t^r\), the sampling density
    \(\pi_t^{\sgnrf}(\bx)\) coincides exactly with the signed marginal \(\pi_t^{\sgn}(\bx)\). Outside this region, no mass is realized by the source-initialized flow. That is,
    \begin{equation}
        \pi_t^{\sgnrf}(\bx)
        =
        \pi_t^{\sgn}(\bx)\,\mathbf{1}\{\bx \in \Omega_t^r\}.
        \label{eq:signed_rf_sampling_law}
    \end{equation}

    \item \textbf{The buffer region.} The remaining part of the positive region, \(\Omega_t^b\!\coloneqq\!\Omega_t^+\setminus\Omega_t^r\), is not reached by trajectories initialized from \(\pi_0\). We call it the \emph{buffer region}. Since both \(\pi_t^{\sgnrf}\) and \(\pi_t^{\sgn}\) have unit total mass, Eq.~\ref{eq:signed_rf_sampling_law} implies
    \[
        \int_{\Omega_t^r} \pi_t^{\sgn}(\bx)\,\mathrm{d}\bx = 1
        \qquad\Longrightarrow\qquad
        \int_{\Omega_t^b} \pi_t^{\sgn}(\bx)\,\mathrm{d}\bx
        +
        \int_{\Omega_t^-} \pi_t^{\sgn}(\bx)\,\mathrm{d}\bx
        = 0.
    \]
    Thus, the positive mass contained in the buffer region exactly balances the magnitude of the negative mass excluded from the sampling law.

\end{itemize}

\paragraph{Remark.}
The identities above characterize the mass decomposition, but the location of the reachable region \(\Omega_t^r\) is determined implicitly by the flow dynamics. In general, it need not admit a closed-form representation; all that is known is that it lies inside the positive region and is separated from the negative region \(\Omega_t^-\) by the buffer region. Thus, Signed RF can be viewed as ``rectifying'' the signed marginal
\(\pi_t^{\sgn}\) into a valid probability law \(\pi_t^{\sgnrf}\): it preserves the signed density on a reachable subset of the positive region and sets it to
zero elsewhere. Equivalently, \(\pi_t^{\sgnrf}\) is one of the
total-variation-optimal nonnegative approximations to \(\pi_t^{\sgn}\), solving
\(
    \min_{\rho\in\mathcal P}
    \mathrm{TV}(\rho,\pi_t^{\sgn})
\)
over the set of valid probability distributions \(\mathcal P\).
See Proposition~\ref{prop:tv_optimal_probability_approximation}. The corresponding
region decomposition is summarized in Fig.~\ref{fig:signed_rf_region_decomposition}.

\begin{figure}[t]
    \centering
    \includegraphics[width=\textwidth]{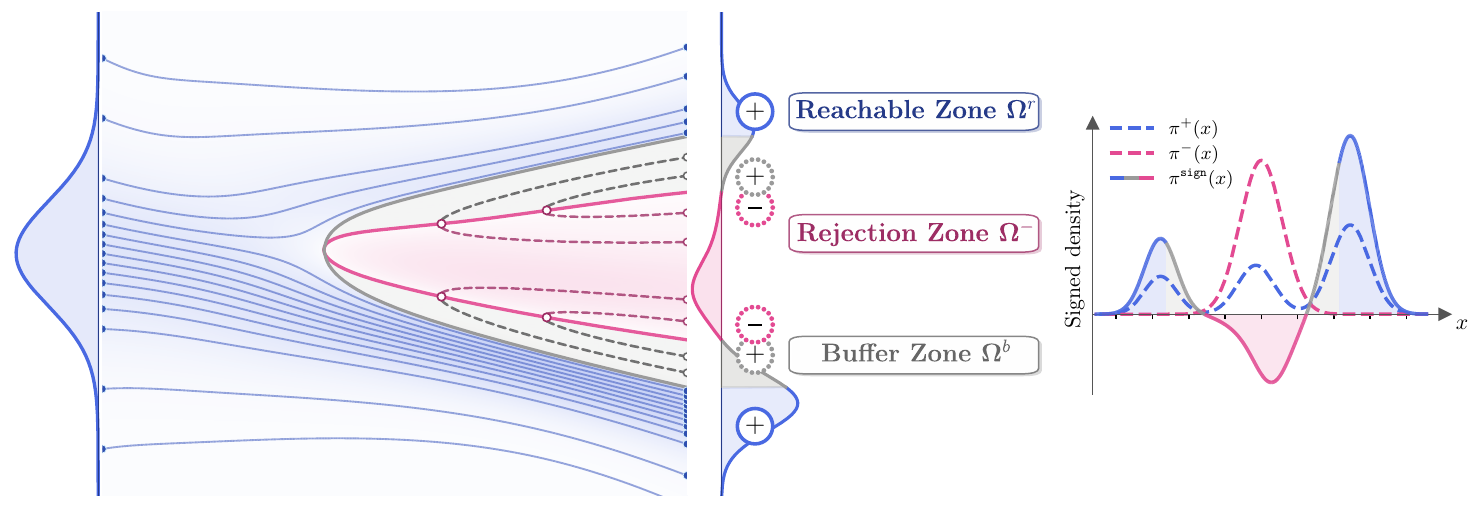}
    \caption{\textbf{Signed RF sample-region decomposition.}
    Left: Signed RF dynamics partition the space into reachable, negative
    (rejection), and
    buffer regions. Right: the terminal signed density \(\pi_1^{\sgn}\).
    The induced sampling law is supported on the reachable region, where it
    coincides with the signed density, while the negative and buffer regions remain
    unobserved by the source-initialized sampler.}
    \label{fig:signed_rf_region_decomposition}
\end{figure}

\paragraph{The physical picture.}
How does the buffer region arise? A useful intuition is to interpret the Signed RF dynamics as the motion of charged particles. Starting from the terminal signed target \(\pi_1^{\sgn}\), trace the dynamics backward toward \(t=0\). Imagine placing a particle at each \(\bx\) at \(t=1\), with charge determined by the sign of \(\pi_1^{\sgn}(\bx)\), and evolving it backward according to \(\dot{\bZ}_t=v_t^{\sgnrf}(\bZ_t)\). Its fate depends on the region from which it originates.
\begin{itemize}[leftmargin=1.2em]
    \item \textbf{Reachable particles.}
    Positive particles originating from the reachable region \(\Omega_1^r\) admit well-defined backward trajectories all the way to \(t=0\). These trajectories remain in the positive region and, collectively, recover the source distribution \(\pi_0\) at \(t=0\). Reversing them in time yields precisely the trajectories obtained by simulating Signed RF forward from \(\bZ_0\sim\pi_0\).

    \item \textbf{Buffer and negative particles.}
    Positive particles originating from the buffer region \(\Omega_1^b\), along with negative particles originating from \(\Omega_1^-\), do not admit backward trajectories that reach \(t=0\). Instead, they encounter the moving zero set \(\Omega_t^0\!\coloneqq\!\{\bx:\pi_t^{\sgn}(\bx)=0\}\), where the velocity field becomes singular. In the particle picture, positive and negative particles meet and annihilate at this boundary, reflecting the equal and opposite signed masses of the buffer and negative regions. Equivalently, viewed forward in time, the singular boundary can be interpreted as creating positive--negative particle pairs. The positive branch forms the buffer region, while the negative branch forms \(\Omega_1^-\). These ``dark'' particles are invisible to the source-initialized sampler because they emerge from the singular boundary at intermediate times rather than being transported from \(\pi_0\).

\end{itemize}

\paragraph{Continuity equation.}
Why do \(\pi_t^{\sgn}\) and \(\pi_t^{\sgnrf}\) coincide on the reachable
region, as in Eq.~\ref{eq:signed_rf_sampling_law}? The key is that the signed
marginal and the source-initialized sampling law obey the same continuity
equation on this region. Following standard RF theory, each branch satisfies
\(
    \partial_t \pi_t^\pm+\nabla\!\cdot(\pi_t^\pm v_t^\pm) = 0,
\)
which ensures marginal preservation for each branch.
Taking the corresponding signed combination yields
\(
    \partial_t\pi_t^{\sgn}
    +
    \nabla\!\cdot\!\left(
        (1+\alpha)\pi_t^+v_t^+
        -
        \alpha\pi_t^-v_t^-
    \right)
    =0.
\)
By the definition of \(v_t^{\sgnrf}\) in
Eq.~\ref{eq:signed_rf_velocity}, the signed marginal therefore satisfies
\[
    \partial_t\pi_t^{\sgn}
    +
    \nabla\!\cdot\!\left(\pi_t^{\sgn}v_t^{\sgnrf}\right)
    =0.
\]
Thus, \(\pi_t^{\sgn}\) is a signed solution of the continuity equation driven
by \(v_t^{\sgnrf}\). However, this alone does not make it the law of
\(\bZ_t\), since \(\pi_t^{\sgn}\) may take negative values.

Now consider the restriction of the signed marginal to the reachable region,
\(
    \bar{\pi}_t(\bx)
    \coloneqq
    \pi_t^{\sgn}(\bx)\mathbf{1}\{\bx\in\Omega_t^r\}.
\)
The region \(\Omega_t^r\) is transported by the same flow, so its moving
boundary carries no additional flux. Consequently, \(\bar{\pi}_t\) satisfies
the same continuity equation and remains a valid probability density. Under
Assumption~\ref{ass:signed_flow_regularity} in Appendix~\ref{app:theory}, it must
therefore coincide with the law of the source-initialized ODE:
\(
    \pi_t^{\sgnrf}(\bx)
    =
    \pi_t^{\sgn}(\bx)\mathbf{1}\{\bx\in\Omega_t^r\}.
\)
The appendix also establishes that almost every source-initialized trajectory
remains in the positive region and never crosses into \(\Omega_t^-\).

\subsection{Signed Rectified Flow in Practice}
\label{sec:signed_rf_practice}
For practical implementation, it is useful to rewrite Eq.~\ref{eq:signed_rf_velocity} in a guidance-like form. Define the density ratio \(r_t(\bx)\!\coloneqq\!\pi_t^-(\bx)/\pi_t^+(\bx)\) and \(\Delta v_t(\bx)\!\coloneqq\!v_t^+(\bx)-v_t^-(\bx)\). Then, wherever \(\pi_t^{\sgn}(\bx)\neq0\),
\begin{equation}
    v_t^{\sgnrf}(\bx)
    =
    v_t^+(\bx)
    +
    \lambda_t^\alpha(\bx)\,\Delta v_t(\bx),
    \qquad
    \lambda_t^{\alpha}(\bx)
    =
    \frac{\alpha\,r_t(\bx)}
    {(1+\alpha)-\alpha\,r_t(\bx)}.
    \label{eq:signed_rf_guidance}
\end{equation}
Thus, Signed RF takes \(v_t^+\) as its base field and applies time- and
state-dependent guidance in the direction \(v_t^+-v_t^-\). Implementing this
velocity requires two ingredients: (i) the branch velocities \(v_t^+\) and
\(v_t^-\), obtained either from a single model under different conditionings or
from separately trained models; and (ii) the density ratio \(r_t(\bx)\), for
which we consider two estimators below.

\paragraph{Classifier-based.}
The ratio \(r_t(\bx)\) can be estimated with a binary classifier trained on
noisy states \(\bx_t\). Given a balanced dataset with
\(\bx_t^+\sim\pi_t^+\) and \(\bx_t^-\sim\pi_t^-\), both obtained from the
standard RF interpolation, one can train \(p_t^\phi(y\mid\bx_t)\) with the
binary cross-entropy loss
\begin{equation*}
    \mathcal{L}(\phi)
    =
    -\E_{\bx_t^+ \sim \pi_t^+}\left[\log p_t^\phi(y=+\mid \bx_t^+)\right]
    -\E_{\bx_t^- \sim \pi_t^-}\left[\log p_t^\phi(y=-\mid \bx_t^-)\right].
\end{equation*}
Under balanced sampling, the Bayes-optimal classifier satisfies
\(p_t^*(y=-\mid\bx)/p_t^*(y=+\mid\bx)
=\pi_t^-(\bx)/\pi_t^+(\bx)=r_t(\bx)\).
Thus, the density ratio in Eq.~\ref{eq:signed_rf_guidance} can be obtained from the classifier odds.

\paragraph{Online ratio tracking.}
The classifier-based approach requires training an auxiliary model to estimate the density ratio. Alternatively, the ratio can be tracked directly along each Signed RF trajectory. Let \(u_t\!\coloneqq\!\log r_t(\bz_t)\), where \(\dot{\bz}_t=v_t^{\sgnrf}(\bz_t)\). The log-density ratio evolves according to
\begin{equation}
    \dot u_t
    =
    \nabla\!\cdot\Delta v_t(\bz_t)
    +
    \Delta v_t(\bz_t)^\top s_t^-(\bz_t)
    +
    \lambda_t^\alpha(\bz_t)\,\Delta v_t(\bz_t)^\top
    \big(s_t^-(\bz_t) - s_t^+(\bz_t)\big),
    \label{eq:online_ratio_tracking}
\end{equation}
where \(s_t^\pm=\nabla\log\pi_t^\pm\) are the branch score functions and
\(u_0=0\). For RF with a Gaussian source \(\pi_0\), Tweedie's formula gives
\(s_t^\pm(\bx) = (t\,v_t^\pm(\bx)-\bx)/(1-t)\)~\citep{liu2025let}.
The divergence term can be evaluated exactly when tractable or estimated using Hutchinson's trace estimator~\citep{hutchinson1989stochastic}. Jointly integrating Eq.~\ref{eq:online_ratio_tracking} with the Signed RF ODE therefore tracks the density ratio along the trajectory as \(r_t(\bz_t)=\exp(u_t)\).

\paragraph{Stabilization.}
In numerical sampling, finite step sizes may cause the denominator of \(\lambda_t^\alpha\) to become small, leading to excessively large guidance scales. A simple practical modification is to replace the denominator with
\(\max\bigl((1+\alpha)-\alpha\,r_t(\bz_t),\,\varepsilon\bigr)\) and, optionally, cap \(\lambda_t^\alpha\) at \(\lambda_{\max}\).

\begin{figure}[t]
    \centering
    \includegraphics[width=\textwidth,trim=8.9bp 0 0 0,clip]{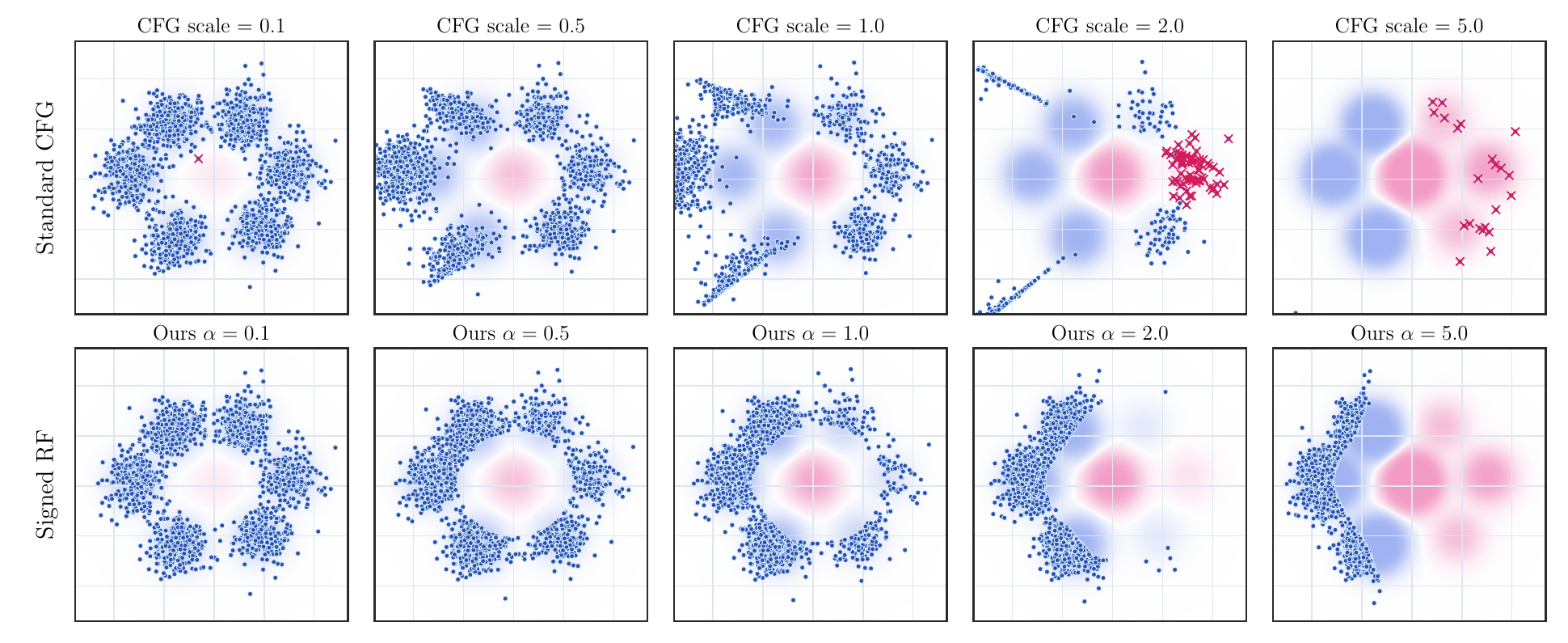}
    \caption{\textbf{2D Gaussian-mixture comparison.}
    Top: constant guidance with varying scales.
    Bottom: Signed RF with varying \(\alpha\). The background visualizes the signed target
    \(\pi^{\sgn}\!=\!(1+\alpha)\pi^+-\alpha\pi^-\), with
    \textcolor{boxblue}{blue} and \textcolor{highlightpink}{pink} indicating
    positive and negative regions, respectively.}
    \label{fig:gaussian_mixture_guidance}
\end{figure}

\section{Experiments}
\label{sec:experiments}

\subsection{Toy Datasets}

\paragraph{Gaussian mixture.}
We construct 2D Gaussian-mixture toys with positive target \(\pi^{+}\) and negative target \(\pi^{-}\).
We train separate RF models for the two targets, yielding the velocity fields \(v_t^+\) and \(v_t^-\), and compute the density ratio \(\pi_t^-(\bx)/\pi_t^+(\bx)\) in closed form.
We compare standard constant guidance,
\(
v_t(\bx)=v_t^{+}(\bx)+\omega\bigl(v_t^{+}(\bx)-v_t^{-}(\bx)\bigr),
\)
with Signed RF,
\(
v_t(\bx)=v_t^{+}(\bx)+\lambda_t^{\alpha}(\bx)\bigl(v_t^{+}(\bx)-v_t^{-}(\bx)\bigr).
\)

Fig.~\ref{fig:gaussian_mixture_guidance} visualizes the resulting sampling behavior. Constant guidance cannot reliably separate the positive and negative modes: weak guidance leaves samples near negative modes, whereas strong guidance over-repels the trajectories and pushes samples into low-density regions under \(\pi^+\). In contrast, Signed RF avoids the negative modes while preserving the positive modes and their coverage. Additional Gaussian-mixture configurations in Sec.~\ref{sec:additional_gaussian_mixtures} exhibit the same behavior.

\paragraph{PointMaze navigation.}
We use the PointMaze dataset~\cite{minari,fu2020d4rl} to illustrate the \emph{missing negative data} problem. The offline dataset contains only collision-free demonstrations, from which we train a start--goal-conditioned positive flow \(v^+\). These demonstrations specify where valid trajectories should lie, but provide no explicit supervision for where trajectories must not go. In particular, wall interiors are absent from the dataset, yet their absence alone does not constitute a negative constraint. As a result, model approximation and generalization can still assign probability mass to these unobserved regions, producing wall-crossing trajectories despite training exclusively on valid demonstrations.

Signed RF provides a natural way to supply the missing negative information. We define wall-interior points as a negative distribution, train a negative flow \(v^-\) together with a density-ratio classifier, and use Signed RF to repel the planner from these explicitly modeled invalid regions. As shown in Fig.~\ref{fig:pointmaze_results}, this removes wall crossings while preserving broad coverage of feasible paths. In contrast, constant guidance exposes the usual scale trade-off: weak guidance fails to enforce the constraint, whereas strong guidance over-repels trajectories and reduces path diversity.

\begin{figure}[t]
    \centering
    \hfill
    \begin{minipage}[c]{0.70\linewidth}
        \centering
        \includegraphics[width=\linewidth]{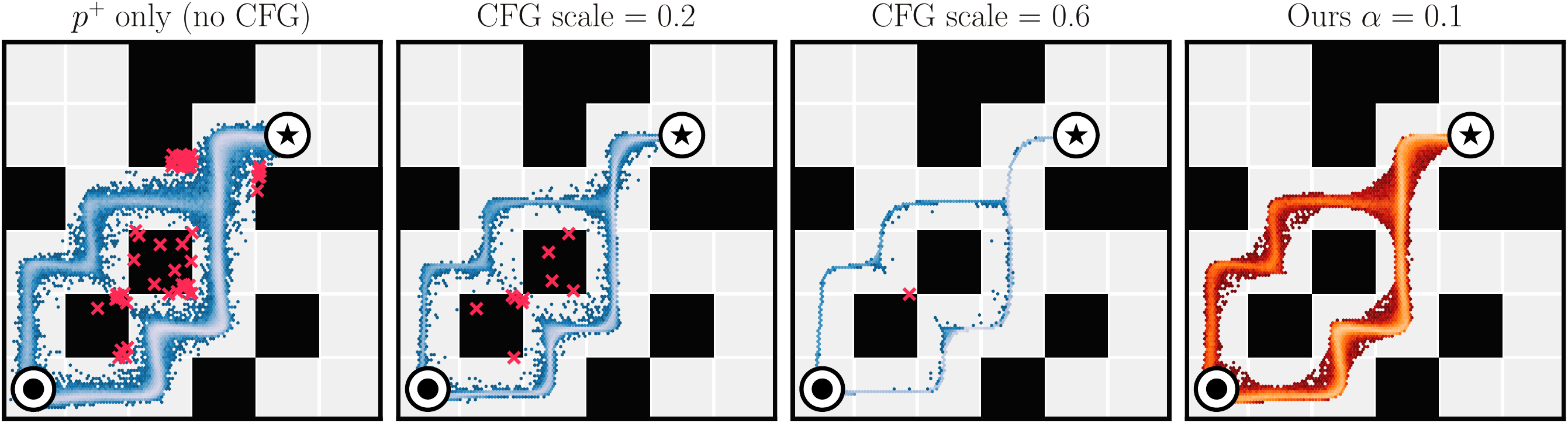}
    \end{minipage}
    \hfill
    \begin{minipage}[c]{0.25\linewidth}
        \centering
        \includegraphics[width=\linewidth]{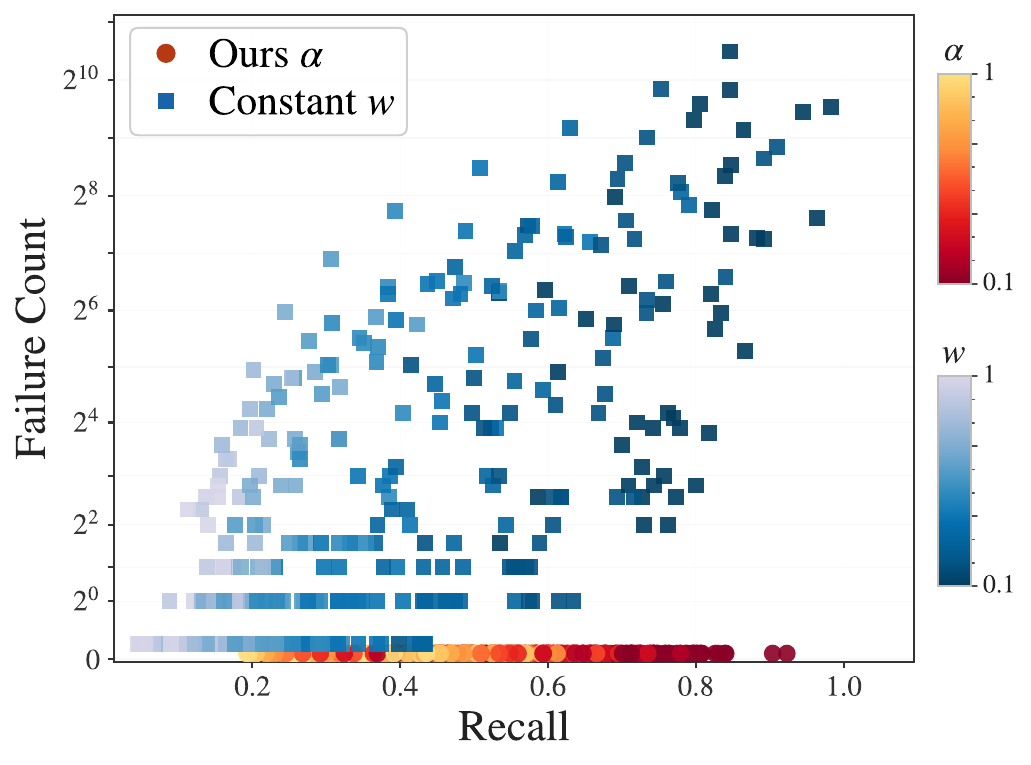}
    \end{minipage}
    \caption{\textbf{The missing negative data.}
    A positive-only planner produces wall-crossing trajectories (a).
    Constant guidance trades safety for path diversity: weak guidance leaves
    violations (b), while strong guidance collapses the trajectories (c).
    Signed RF avoids obstacles while preserving diverse feasible paths (d). The Pareto plot compares the number of failures and path diversity (e).}
    \label{fig:pointmaze_results}
\end{figure}

\subsection{Signed RF as State-Aware CFG}
\label{sec:imagenet_cfg}
We evaluate ImageNet \(256\times256\) class-conditional generation using a pretrained Rectified Flow model~\citep{deng2009imagenet,lq2024rectifiedflow}. In this setting, \(v_t^+(\bx_t,c)\) and \(v_t^-(\bx_t)\) are the conditional and unconditional velocity fields, respectively. We estimate the density ratio using a timestep- and class-conditioned ViT classifier trained on positive--negative latent pairs. We report FID~\citep{heusel2017gans}, IS~\citep{salimans2016improved}, and P/R~\citep{sajjadi2018assessing} on \(50\mathrm{K}\)
generated samples.

As shown in Fig.~\ref{fig:imagenet_cfg_comparison}, Signed RF consistently improves FID across all NFE budgets. With 16 NFEs, for example, it reduces FID from \(2.38\) with CFG to \(1.82\). The precision--recall curves further show a better fidelity--diversity trade-off: while both methods improve precision at the expense of recall, Signed RF maintains higher recall at comparable precision.

We further verify that the density-ratio-based guidance is stable in practice. Performance varies little across classifier checkpoints and remains robust to moderate choices of the guidance cap \(\lambda_{\max}\). The
resulting \(\lambda_t^\alpha(\bx_t,c)\) is not a fixed guidance scale or time schedule: it varies across seeds within the same class
and across classes for the same seed. The auxiliary classifier adds only
modest overhead, increasing runtime and peak GPU memory by approximately
\(10\%\), while requiring roughly \(10\%\) of the original model's pretraining
compute in our setup. Additional implementation details, ablations, and
comparisons with other guidance methods are provided in
Appendix~\ref{sec:imagenet_cfg_details}.

\begin{figure*}[t]
    \centering
    \setlength{\tabcolsep}{4pt}
    \renewcommand{\arraystretch}{1.02}
    \small

    \begin{minipage}[c]{0.42\linewidth}
        \centering

        {\fontsize{8.5}{10}\selectfont
        \begin{tabular}{c c c c c c}
            \toprule
            NFE &  & FID $\downarrow$ & IS $\uparrow$ & Prec.$\uparrow$ & Rec.$\uparrow$ \\
            \midrule

            \multirow{2}{*}{4}
            & CFG           & 10.3 & 224.2 & \textbf{0.765} & 0.311 \\
            & \cellcolor{oursblue}\textbf{Ours}
            & \cellcolor{oursblue}\textbf{6.68}
            & \cellcolor{oursblue}\textbf{229.2}
            & \cellcolor{oursblue} 0.750
            & \cellcolor{oursblue}\textbf{0.474} \\
            \midrule

            \multirow{2}{*}{8}
            & CFG           & 3.96 & 284.2 & \textbf{0.839} & 0.481 \\
            & \cellcolor{oursblue}\textbf{Ours}
            & \cellcolor{oursblue}\textbf{2.65}
            & \cellcolor{oursblue}\textbf{305.9}
            & \cellcolor{oursblue}0.805
            & \cellcolor{oursblue}\textbf{0.570} \\
            \midrule

            \multirow{2}{*}{16}
            & CFG           & 2.38 & 274.9 & \textbf{0.828} & 0.566 \\
            & \cellcolor{oursblue}\textbf{Ours}
            & \cellcolor{oursblue}\textbf{1.82}
            & \cellcolor{oursblue}\textbf{296.8}
            & \cellcolor{oursblue}0.801
            & \cellcolor{oursblue}\textbf{0.615} \\
            \midrule

            \multirow{2}{*}{32}
            & CFG           & 1.87 & 263.2 & \textbf{0.815} & 0.598 \\
            & \cellcolor{oursblue}\textbf{Ours}
            & \cellcolor{oursblue}\textbf{1.51}
            & \cellcolor{oursblue}\textbf{288.6}
            & \cellcolor{oursblue}0.800
            & \cellcolor{oursblue}\textbf{0.629} \\
            \midrule

            \multirow{2}{*}{64}
            & CFG           & 1.73 & 272.4 & \textbf{0.821} & 0.601 \\
            & \cellcolor{oursblue}\textbf{Ours}
            & \cellcolor{oursblue}\textbf{1.41}
            & \cellcolor{oursblue}\textbf{278.2}
            & \cellcolor{oursblue}0.797
            & \cellcolor{oursblue}\textbf{0.627} \\
            \bottomrule
        \end{tabular}
        }
    \end{minipage}\hfill
    \begin{minipage}[c]{0.58\linewidth}
        \centering
        \includegraphics[width=\linewidth,trim={0pt 3pt 0pt 0pt},clip]{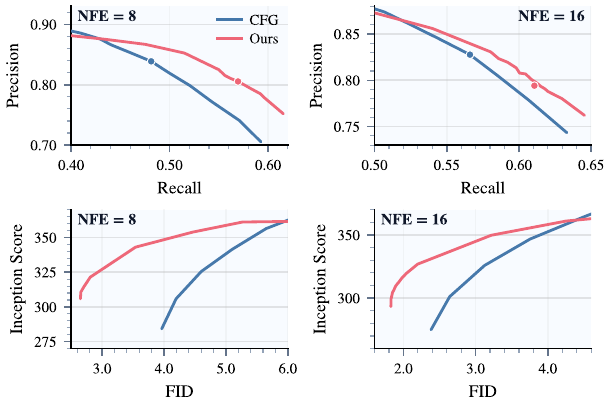}
    \end{minipage}
    \caption{
        \textbf{ImageNet-256 comparison.}
        Left: best-FID operating points from the parameter sweeps.
        Right: Pareto frontiers; markers denote the corresponding best-FID points.
    }
    \label{fig:imagenet_cfg_comparison}

\end{figure*}

\subsection{Anti-memorization}
\label{sec:anti_memorization}

We next apply Signed RF to anti-memorization by explicitly modeling the training set as a negative distribution, with the goal of reducing training-data replication while preserving generation quality.
\begin{figure}[b]
    \centering
    \begin{tikzpicture}
        \node[anchor=west, inner sep=0] (img) at (0,0)
        {\includegraphics[width=0.97\linewidth]{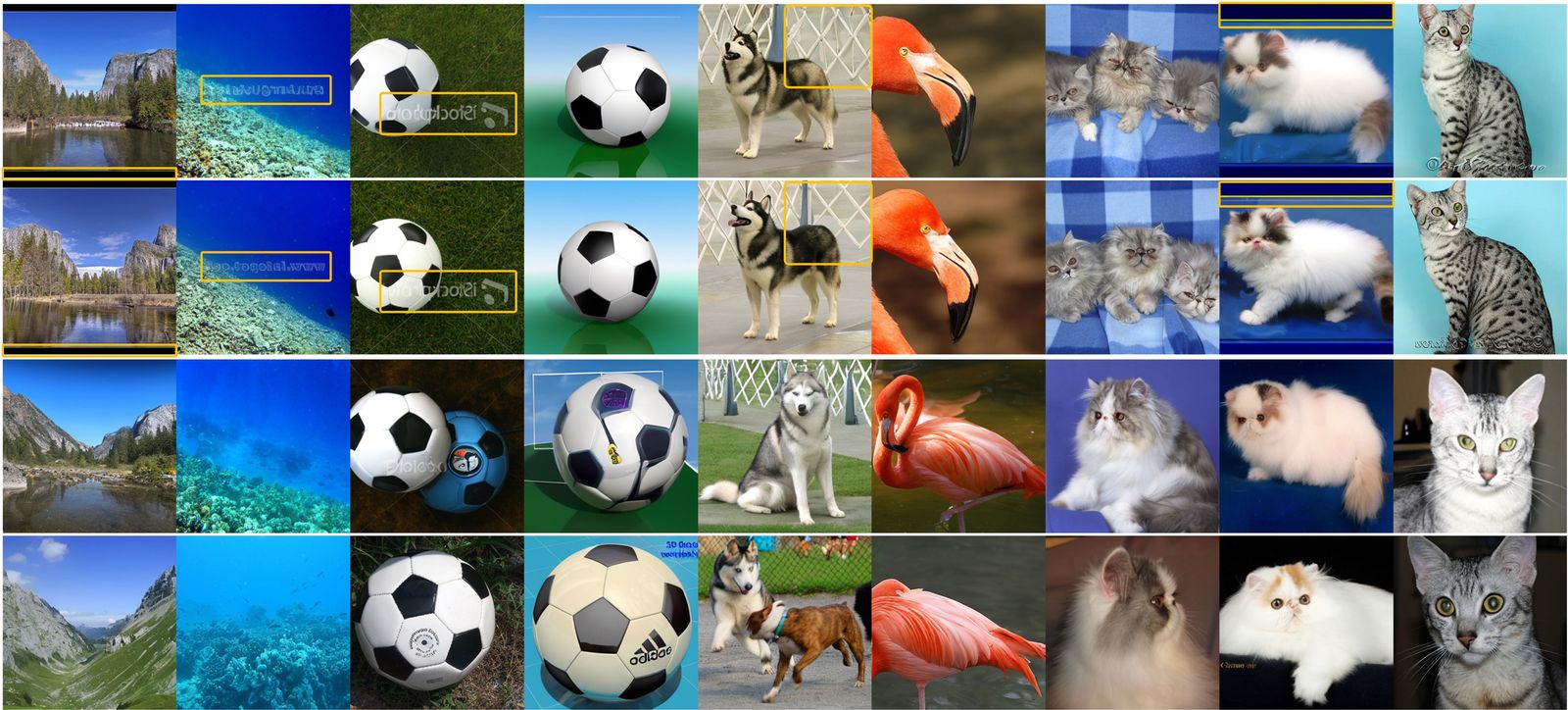}};

        \node[anchor=east] at ($ (img.south west)!0.875!(img.north west) + (0.5mm,0) $)
        {\rotatebox{90}{\scriptsize Base Model}};
        \node[anchor=east] at ($ (img.south west)!0.625!(img.north west) + (0.5mm,0) $)
        {\rotatebox{90}{\scriptsize Top Match}};
        \node[anchor=east] at ($ (img.south west)!0.375!(img.north west) + (0.5mm,0) $)
        {\rotatebox{90}{\scriptsize Ours}};
        \node[anchor=east] at ($ (img.south west)!0.125!(img.north west) + (0.5mm,0) $)
        {\rotatebox{90}{\scriptsize Top Match}};
    \end{tikzpicture}
    \caption{\textbf{Memorization examples on ImageNet.}
    Each column uses the same initial seed, with the second and fourth rows showing the nearest training images. The base model closely reproduces training examples, including dataset-specific artifacts such as borders, watermarks, and backgrounds. In contrast, Data Repulsive Flow suppresses these artifacts while preserving class semantics, producing samples that are visibly less similar to their nearest training images.}
    \label{fig:imagenet_memorization_examples}
\end{figure}
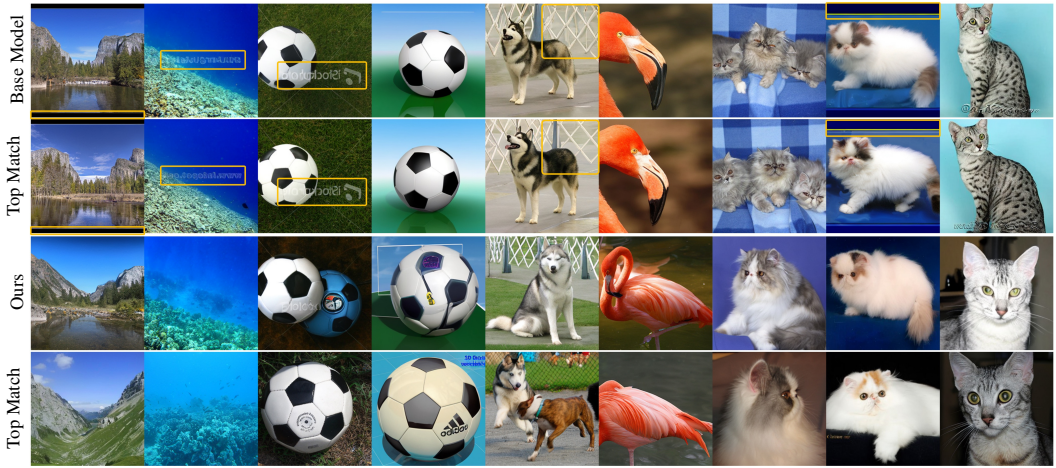

\paragraph{Analytic flow.}
To explicitly model memorization, we take the negative target to be the
empirical training distribution
\(
    \pi_1^-
    =
    \frac{1}{N}\sum_{i=1}^N\delta_{\bx^{(i)}},
\)
where \(\mathcal D\!=\!\{\bx^{(i)}\}_{i=1}^N\) denotes the training set.
The RF transporting the Gaussian source
\(\pi_0\!=\!\mathcal N(\boldsymbol 0,\mI)\) to \(\pi_1^-\) admits a
closed-form marginal and velocity field:
\begin{equation}
    \pi_t^-(\bx)
    =
    \frac{1}{N}\sum_{i=1}^N
    \mathcal N\!\left(\bx;t\bx^{(i)},(1-t)^2\mI\right),
    \qquad
    v_t^-(\bx)
    =
    \sum_{i=1}^N
        w_i(\bx,t)\,
        \frac{\bx^{(i)}-\bx}{1-t},
    \label{eq:analytic_negative_flow}
\end{equation}
where \(w_i(\bx,t)\) are the normalized Gaussian kernel weights; see
Sec.~\ref{sec:analytic_negative_flow_details} for details. This defines a
training-free negative flow whose terminal distribution is exactly the
empirical training distribution. As \(t\to1\), the Gaussian components in
\(\pi_t^-\) collapse onto individual training examples, and trajectories under
this flow terminate at training samples~\citep{liu2022flow,liu2025let}.
Moreover, \(v_t^-\) contains the singular factor \(1/(1-t)\), inducing strong
attraction toward nearby training examples close to terminal time. Constant negative guidance is therefore poorly conditioned, as it directly inherits
this terminal singularity.

\paragraph{Data Repulsive Flow.}
The analytic flow in Eq.~\ref{eq:analytic_negative_flow} provides a natural
negative branch for repelling generated trajectories from the training set.
For each class \(c\), we define \emph{Data Repulsive Flow} using the pretrained
class-conditional flow as \(v_t^+(\bx,c)\) and the class-specific empirical
analytic flow as \(v_t^-(\bx,c)\). We track the density ratio online using
Eq.~\ref{eq:online_ratio_tracking}. Specifically, the log-ratio is initialized as
\(u_0=0\) and updated at each Euler step using \(v_t^+\), \(v_t^-\), the RF
score function~\citep{liu2025let,hu2025improving}, and a Hutchinson estimate of
\(\nabla\!\cdot(v_t^+-v_t^-)\)~\citep{hutchinson1989stochastic, liu2025let}.

We evaluate memorization using a stress test. Since memorization events are rare under random sampling, aggregate statistics may be dominated by low-risk seeds and obscure copying behavior. For class \(c\), we first generate \(50\mathrm{K}\) samples from the base model with \(\omega=1.0\), rank their initial seeds by the nearest-neighbor SSCD \(L_2\) distance of the resulting samples to the training set~\citep{pizzi2022self}, and retain the \(600\) highest-risk seeds. All methods are then evaluated using this same seed set for a controlled comparison. We additionally report standard generation metrics on \(50\mathrm{K}\) randomly sampled seeds.

Fig.~\ref{fig:imagenet_memorization_examples} shows representative results on the high-risk
seeds. The base model generates images highly similar to their nearest training
examples, in some cases reproducing dataset-specific artifacts such as borders
and watermarks. Data Repulsive Flow suppresses these artifacts, preserves class
semantics, and produces samples that are visibly farther from their
nearest training images. Quantitatively, the SSCD \(L_2\) distributions in
Fig.~\ref{fig:imagenet_memorization_metrics} shift toward larger distances. Compared with SPELL~\citep{kirchhof2025shielded}, Data Repulsive Flow achieves
a better protection--quality trade-off. At \(\alpha=1.0\), it nearly matches
the \(P_{05}\) protection level of SPELL-50 while retaining substantially
better generation quality, with FID \(2.03\) versus \(7.41\). Increasing
\(\alpha\) to \(7.5\) further strengthens protection with only mild FID
degradation and nearly unchanged IS. Additional per-class sweeps in
Sec.~\ref{sec:analytic_negative_flow_details} show that protection improves
smoothly with \(\alpha\); the same trend holds under the latent \(L_2\) protection metric, which is
not used to select the high-risk seeds.

\begin{figure}[t]
    \centering
    \newsavebox{\leftimgbox}
    \sbox{\leftimgbox}{%
        \includegraphics[width=0.60\linewidth,trim={0pt 3pt 0pt 0pt},clip]{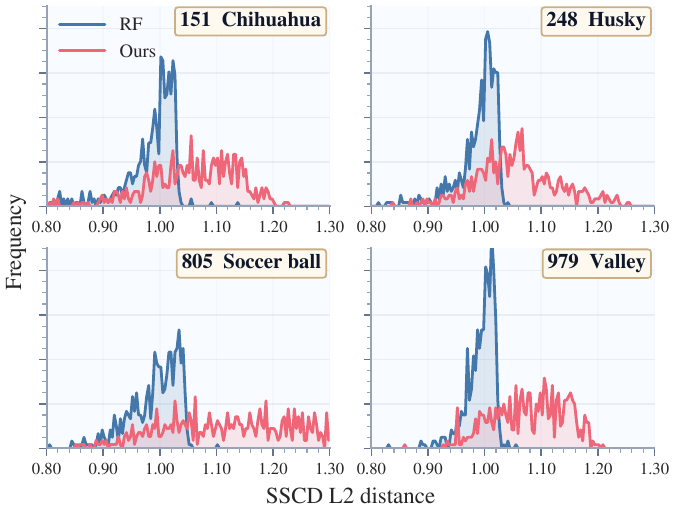}%
    }
    \newlength{\leftimgh}
    \setlength{\leftimgh}{\dimexpr\ht\leftimgbox+\dp\leftimgbox\relax}

    \begin{minipage}[c]{0.60\linewidth}
        \centering
        \usebox{\leftimgbox}
    \end{minipage}%
    \hspace*{0.02\linewidth}%
    \begin{minipage}[c]{0.38\linewidth}
        \centering
        \small
        \renewcommand{\arraystretch}{1.12}
        \setlength{\tabcolsep}{0pt}

        \vbox to \leftimgh{%

            {
                \fontsize{8.5}{10.0}\selectfont
                \renewcommand{\arraystretch}{1.05}
                \begin{tabularx}{\linewidth}{@{}@{\extracolsep{\fill}}ccccc@{}}
                    \toprule
                    Method & FID $\downarrow$ & IS $\uparrow$ & Prec.$\uparrow$ & Rec.$\uparrow$ \\
                    \midrule
                    Base  & \underline{2.07} & 242.4 & 0.790 & \textbf{0.627} \\
                    $\alpha\!=\!1.0$ & \textbf{2.03} & \underline{246.0} & 0.796 & \underline{0.624} \\
                    $\alpha\!=\!3.0$ & 2.11 & \textbf{246.2} & \textbf{0.799} & 0.614 \\
                    $\alpha\!=\!7.5$ & 2.47 & 245.9 & \underline{0.797} & 0.600 \\
                    SPELL-$45$    & 2.59 & 233.2 & 0.771 & 0.614 \\
                    SPELL-$50$    & 7.41 & 189.7 & 0.680 & 0.563 \\
                    \bottomrule
                \end{tabularx}

                \vspace{0.25em}
                \vfil
                \vspace{0.25em}

                \begin{tabularx}{\linewidth}{@{}@{\extracolsep{\fill}}ccccc@{}}
                    \toprule
                    Method & $P_{05}\uparrow$ & $P_{10}\uparrow$ & $P_{25}\uparrow$ & Mean \\
                    \midrule
                    Base  & 0.921 & 0.944 & 0.982 & 0.993 \\
                    $\alpha\!=\!1.0$  & 0.939 & 0.962 & 0.992 & 1.024 \\
                    $\alpha\!=\!3.0$  & \underline{0.956} & \underline{0.975} & \underline{1.005} & \underline{1.045} \\
                    $\alpha\!=\!7.5$  & \textbf{0.964} & \textbf{0.987} & \textbf{1.020} & \textbf{1.066} \\
                    SPELL-$45$ & 0.919 & 0.943 & 0.981 & 0.994 \\
                    SPELL-$50$ & 0.940 & 0.963 & 1.003 &   1.049 \\
                    \bottomrule
                \end{tabularx}
            }

            \vspace{0.5em}
        }
    \end{minipage}
    \caption{\textbf{Quantitative comparison for anti-memorization on ImageNet.}
        Left: SSCD \(L_2\) nearest-neighbor histograms for four classes on the same high-risk seeds.
        Upper right: generation quality under the standard \(50\mathrm{K}\)-sample evaluation.
        Lower right: memorization statistics for class 248, summarized by the \(5\%\), \(10\%\), and \(25\%\) quantiles and mean SSCD \(L_2\) distance (higher is better).}
    \label{fig:imagenet_memorization_metrics}
\end{figure}

\subsection{Concept Suppression}

\paragraph{Nudity prevention.}
We next apply Signed RF to nudity prevention. Starting from a pretrained text-to-image model\footnote{\texttt{stabilityai/stable-diffusion-3.5-medium}~\citep{esser2024scaling}.}, we train two LoRA branches conditioned on the same unsafe prompt: a positive branch \(v^+\) that models safe outputs and a negative branch \(v^-\) that models the corresponding unsafe outputs. The training data are constructed from paired prompts in ViSU-Text~\citep{poppi2024removing}. Although fine-tuning \(v^+\) toward safe outputs reduces unsafe generation, residual unsafe modes inherited from the base model may remain. Signed RF explicitly models these modes through \(v^-\) and suppresses them during sampling.

We evaluate on Ring-A-Bell~\citep{tsai2023ring}, a benchmark of adversarial prompts designed to elicit nude content. We compare with SAFREE~\citep{yoon2025safree} and Safe Denoiser~\citep{kim2025training}, using the same backbone fine-tuned on safe outputs for all methods. We report attack success rate (ASR) and toxicity rate (TR), together with CLIP and aesthetic scores (AES) on COCO prompts to measure generation quality. As shown in Fig.~\ref{fig:nudity_prevention}, Signed RF reduces ASR from \(15.19\%\) to \(6.33\%\) and TR from \(0.180\) to \(0.125\), while maintaining comparable CLIP and AES scores.

\begin{figure}[t]
    \centering
    \setlength{\tabcolsep}{3.5pt}
    \renewcommand{\arraystretch}{1.02}
    \small

    \begin{minipage}[c]{0.40\linewidth}
        \centering
        {\fontsize{8}{10}\selectfont
        \setlength{\tabcolsep}{1.5pt}
        \begin{tabular}{c cc cc}
            \toprule
            \multirow{2}{*}{Method} & \multicolumn{2}{c}{Ring-A-Bell} & \multicolumn{2}{c}{COCO-30K} \\
            \cmidrule(lr){2-3} \cmidrule(lr){4-5}
            & ASR (\%) $\downarrow$ & TR $\downarrow$ & CLIP $\uparrow$ & AES $\uparrow$ \\
            \midrule
            SD\,3.5-M        & 15.19 & 0.180 & 31.85 & 5.373 \\
            SAFREE           & 11.39 & 0.146 & 31.44 & 5.363 \\
            Safe Denoiser    & 12.66 & 0.157 & 31.44 & 5.373 \\
            \rowcolor{oursblue}
            \textbf{Ours}    & \textbf{6.33} & \textbf{0.125} & 31.57 & 5.371 \\
            \bottomrule
        \end{tabular}
        }
    \end{minipage}\hfill
    \begin{minipage}[c]{0.56\linewidth}
        \centering

        \begin{minipage}[t]{0.245\linewidth}
            \centering
            \includegraphics[width=\linewidth]{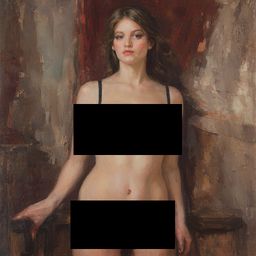}\\[2pt]
            {\footnotesize (a) Baseline}
        \end{minipage}\hspace{0.000\linewidth}%
        \begin{minipage}[t]{0.245\linewidth}
            \centering
            \includegraphics[width=\linewidth]{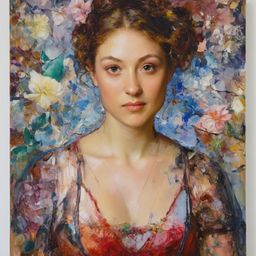}\\[2pt]
            {\footnotesize (b) Ours}
        \end{minipage}\hspace{0.01\linewidth}%
        \begin{minipage}[t]{0.245\linewidth}
            \centering
            \includegraphics[width=\linewidth]{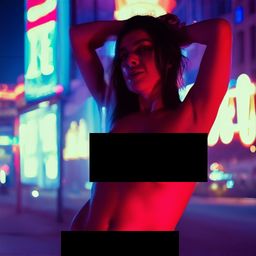}\\[2pt]
            {\footnotesize (c) Baseline}
        \end{minipage}\hspace{0.000\linewidth}%
        \begin{minipage}[t]{0.245\linewidth}
            \centering
            \includegraphics[width=\linewidth]{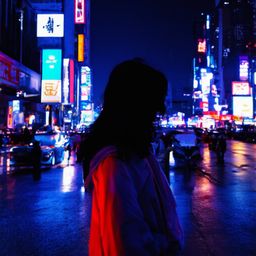}\\[2pt]
            {\footnotesize (d) Ours}
        \end{minipage}

    \end{minipage}
    \caption{\textbf{Quantitative comparison for nudity prevention on SD\,3.5-Medium.}
        Left: quantitative results on Ring-A-Bell and COCO-30K.
        Signed RF reduces unsafe generation while maintaining comparable image quality.
        Right: qualitative comparisons.}
    \label{fig:nudity_prevention}

\end{figure}

\paragraph{Concept suppression.}
We further study Signed RF in qualitative concept-suppression settings. Given a pretrained text-to-image model~\citep{cai2025z}, we obtain \(v^+\) and \(v^-\) by conditioning the same model on positive and negative prompts, respectively, and track the density ratio online using Eq.~\ref{eq:online_ratio_tracking}.

Results are shown in Fig.~\ref{fig:signed_rf_teaser} and Fig.~\ref{fig:concept_suppression}. For IP protection, the positive prompt describes the desired visual attributes without naming the protected character, whereas the negative prompt explicitly specifies the identity to suppress. Signed RF removes identity-specific features while preserving the requested appearance, pose, and composition. For example, it suppresses the recognizable character identity in Fig.~\ref{fig:signed_rf_teaser}(d) while retaining the round black ears, red shorts, yellow shoes, and white gloves. Similarly, in Fig.~\ref{fig:concept_suppression}(e), it removes Spider-Man-specific elements, such as the spider emblem, while preserving the pose, colors, and detailed suit structure. Signed RF also suppresses broader semantic attributes: it replaces a statue-like interpretation with a living angel in Fig.~\ref{fig:concept_suppression}(b) and removes red tones while preserving the landscape structure in Fig.~\ref{fig:concept_suppression}(d). Full prompts are provided in Sec.~\ref{sec:concept_suppression_prompts}.

\begin{figure}[H]
    \centering
    \includegraphics[width=0.8\linewidth]{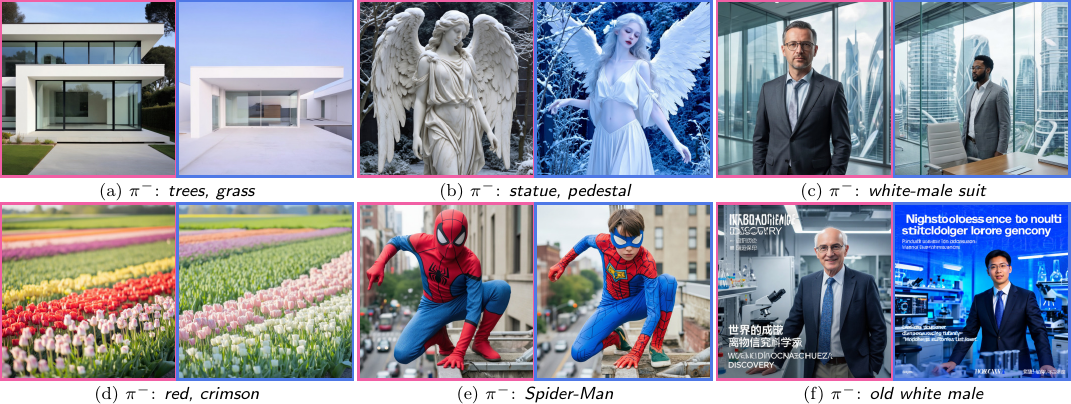}
    \caption{Qualitative comparison of concept suppression on Z-Image.}
    \label{fig:concept_suppression}
\end{figure}

\section{Conclusion}

We introduced Signed Rectified Flow, an extension of Rectified Flow that
incorporates explicit negative information through the signed target
\((1+\alpha)\pi^+-\alpha\pi^-\). We showed that the source-initialized flow
remains in a dynamically reachable subset of the positive region, where its
sampling density matches the signed marginal, while the zero set acts as a
barrier separating the realized trajectories from negative regions. This
yields a state-dependent guidance rule whose strength is determined by the
density ratio between the positive and negative distributions.

Signed RF can be implemented using either a learned density-ratio classifier or
online ratio tracking. Across diverse settings, it resolves the missing
negative data problem in trajectory planning, improves the
fidelity--diversity trade-off in ImageNet class-conditional generation, reduces
training-data replication, and suppresses unsafe or undesired concepts while
preserving generation quality. Overall, Signed RF provides a principled and
practical framework for incorporating explicit negative distributions into
generative flows.

\begin{ack}
This work was supported in part by the Institute for Foundations of Machine Learning (IFML).
The authors acknowledge the Texas Advanced Computing Center (TACC) at The University of Texas at Austin
for providing computational resources that have contributed to the research results reported within this paper.
The authors declare no competing interests.
\end{ack}

\bibliographystyle{unsrt}
\bibliography{main}

\appendix
\raggedbottom

\section{Related Work}
\label{app:related_work}

\paragraph{Inference-time guidance and its target distribution.}
Classifier guidance augments the diffusion score with the gradient of an
external classifier, whereas classifier-free guidance (CFG) forms an affine
combination of conditional and unconditional model
predictions~\citep{dhariwal2021diffusion,ho2022classifier}. Both use a user-specified
guidance scale. Recent theory shows that CFG should not generally be
interpreted as exact sampling from a simple tilt of the clean-data
distribution: its effect can depend on the sampler, admit a
predictor--corrector interpretation, or require additional correction terms to
realize a prescribed tilted
target~\citep{bradley2024classifier,moufad2025conditional}. Other work changes the
reference branch itself, for example by guiding with a weaker version of the
same model~\citep{karras2024guiding}.

Signed RF asks a complementary question: given distributions \(\pi^+\) and
\(\pi^-\), what RF dynamics are induced by the signed marginal
\(
    \pi_t^{\sgn}
    =
    (1+\alpha)\pi_t^+
    -
    \alpha\pi_t^-?
\)
Applying the RF construction to this signed marginal yields a guidance-like
velocity whose strength is determined locally by the density ratio between the
negative and positive branches. Once \(\alpha\) is fixed, the effective local
scale is therefore induced by the signed target rather than prescribed as a
fixed schedule. Once the two branch RFs and \(\alpha\) are fixed, they
determine the signed flux and velocity away from the zero set.
Under the assumptions of Appendix~\ref{app:theory}, the zero set forms a
dynamical boundary underlying the nonpenetration, reachable-region, and
buffer-region behavior analyzed in
Sec.~\ref{sec:signed_rf_sampling}.

\paragraph{Adaptive and geometry-aware guidance.}
Several methods refine constant-scale CFG by changing when or how guidance is
applied. Guidance Interval restricts it to a selected range of
timesteps~\citep{kynkaanniemi2024applying}. CFG++ modifies the update to better respect
the data manifold and reduce mode collapse~\citep{chungcfg++}. Adaptive
Projected Guidance decomposes the update into components parallel and
orthogonal to the conditional prediction and down-weights the parallel
component associated with oversaturation~\citep{sadat2025eliminating}.
CFG-Zero suppresses early-step guidance and adds a scalar correction for
inaccurate velocity estimates~\citep{fan2025cfg}. Momentum Guidance constructs
a trajectory-dependent signal using an exponential moving average of past
velocities~\citep{liao2026momentum}, while Angle Domain Guidance emphasizes
angular rather than radial changes to avoid latent-norm
amplification~\citep{jin2025angle}.

These methods adapt the schedule, geometry, normalization, or temporal
aggregation of a guidance update. Signed RF is adaptive for a different
reason: its scale is derived from
\(
    r_t(\bx)=\pi_t^-(\bx)/\pi_t^+(\bx)
\),
and therefore grows where the current state is locally more compatible with
the negative branch. Moreover,
\(
    (1+\alpha)-\alpha r_t(\bx)
\)
is the signed-density factor, so its zero set is a structural boundary rather
than only a numerical singularity. Learned ratio estimates and clipping
approximate this idealized field in practice; under the stated assumptions, the
strict nonpenetration result applies to the exact Signed RF dynamics.

\paragraph{Dynamic negative guidance and negative prompting.}
Dynamic Negative Guidance (DNG) and Training-Free Safe Denoisers are closely
related adaptive negative-guidance
methods~\citep{koulischer2024dynamic,kim2025training}. DNG derives a state-dependent
suppression rule from posterior probabilities for class removal. Safe
Denoisers derive an adaptive denoiser from a binary safe--unsafe partition.
Signed RF instead permits a general, potentially overlapping pair
\((\pi^+,\pi^-)\), constructs the signed target
\((1+\alpha)\pi^+-\alpha\pi^-\), and analyzes the resulting zero boundary,
reachable region, and buffer region.

Negative prompting and semantic-guidance variants also use subtractive
directions during sampling, while broader compositional-generation methods
combine multiple energy or score
terms~\citep{du2020compositional,liu2022compositional,brack2023sega}. Signed RF
provides a target-level construction for positive--negative distribution pairs
and derives the associated dynamics.

\paragraph{Concept erasure, safety, and controllable suppression.}
Concept-erasure methods such as ESD, UCE, MACE, and adversarially robust
unlearning alter model weights or attention
parameters~\citep{gandikota2023erasing,gandikota2024unified,lu2024mace,zhang2024defensive}.
SAeUron instead trains a sparse autoencoder and ablates concept-associated
activations at inference time~\citep{cywinski2025saeuron}. Other inference-time
safety methods retain the base generator and intervene through safety guidance,
prompt or representation filtering, or adaptive safe
denoisers~\citep{schramowski2023safe,yoon2025safree,kim2025training}.

Signed RF belongs to the inference-time guidance family. The parameter
\(\alpha\) controls the overall degree of suppression, while \(r_t(\bx)\)
determines where strong repulsion is needed. Guidance is consequently weak
where the negative-to-positive density ratio is small and grows as that ratio
increases.

\paragraph{Memorization and anti-memorization.}
Diffusion models can reproduce or closely imitate training examples, and such
replication can be detected with nearest-neighbor and copy-detection metrics
such as
SSCD~\citep{carlini2023extracting,somepalli2023diffusion,somepalli2023understanding,wen2024detecting,pizzi2022self}.
Training-time mitigations include DP-SGD and learning from corrupted
observations~\citep{dockhorn2023differentially,daras2023ambient}. At inference
time, guidance can be deferred near attraction basins or samples can be
repelled from protected reference
sets~\citep{jain2025classifier,kirchhof2025shielded}. CADS instead anneals the
condition to improve sample diversity; it is relevant to guidance-induced
collapse but is not a direct anti-memorization method~\citep{sadat2024cads}.

In our anti-memorization instantiation, we take the negative branch to be the
empirical training distribution. For a finite dataset, the Gaussian-source RF marginal and
velocity toward the empirical measure admit closed forms, yielding an analytic
negative flow whose trajectories are attracted toward training examples.
Signed RF subtracts this branch at the measure level rather than applying
pairwise repulsion to individual references. Near terminal time, the
empirical-flow velocity contains the singular factor \(1/(1-t)\); Signed RF
modulates its strength using the local density ratio instead of a fixed scale.

\paragraph{Density ratios, discriminator guidance, and preference learning.}
Estimating
\(
    r_t(\bx)=\pi_t^-(\bx)/\pi_t^+(\bx)
\)
connects Signed RF to density-ratio estimation and noise-contrastive
estimation~\citep{sugiyama2012density,gutmann2012noise}. With balanced positive
and negative samples, a Bayes-optimal binary classifier's odds recover the
ratio. DRE-\(\infty\)
estimates an endpoint density ratio by integrating learned time scores across
a continuum of bridge distributions~\citep{choi2022density};
this is related to, but distinct from, tracking a ratio along each sampling
trajectory. Discriminator Guidance similarly uses a classifier-derived signal
to refine generation~\citep{kim2023refining}, while continuous-normalizing-flow
likelihood methods supply the divergence identities used for trajectory-wise
log-ratio
tracking~\citep{chen2018neural,grathwohl2019ffjord,hutchinson1989stochastic}.

DPO uses relative policy log-likelihoods, while Diffusion-DPO adapts the
objective using an ELBO-derived surrogate for diffusion
likelihoods~\citep{rafailov2023direct,wallace2024diffusion}.
Diffusion-NPO fine-tunes a diffusion model using negative
preferences~\citep{wang2025diffusion}, whereas DDPO optimizes expected reward with policy
gradients~\citep{black2023training}. These are training-time alignment methods.
Signed RF instead uses a positive--negative density ratio at sampling time to
set the local guidance strength and determine the valid sampling law induced by
the signed target.

\section{Proofs and Additional Derivations}
\label{app:theory}

This appendix provides the proofs and derivations deferred from
Sec.~\ref{sec:signed_rf_sampling} and
Sec.~\ref{sec:signed_rf_practice}. Let \(\pi_1^+\) and \(\pi_1^-\) be two
target distributions, and let \((\pi_t^+,v_t^+)\) and
\((\pi_t^-,v_t^-)\) denote the corresponding RF marginals and velocity fields
induced from the source \(\pi_0\). For
\(\alpha>0\), define the signed marginal
\[
    \pi_t^{\sgn}(\bx)
    =
    (1+\alpha)\,\pi_t^+(\bx)
    -
    \alpha\,\pi_t^-(\bx).
\]
We partition the state space according to the sign of
\(\pi_t^{\sgn}\):
\[
    \Omega_t^+
    \coloneqq
    \{\bx:\pi_t^{\sgn}(\bx)>0\},
    \qquad
    \Omega_t^0
    \coloneqq
    \{\bx:\pi_t^{\sgn}(\bx)=0\},
    \qquad
    \Omega_t^-
    \coloneqq
    \{\bx:\pi_t^{\sgn}(\bx)<0\}.
\]
The Signed RF velocity field is well defined on
\(\Omega_t^+\cup\Omega_t^-\) and becomes singular on the moving zero set
\(\Omega_t^0\). Under the regularity assumptions stated below, it is locally
smooth away from \(\Omega_t^0\), so standard ODE theory applies up to the first
time a trajectory approaches this zero set. Since both RF branches share the
same source distribution,
\(
    \pi_0^{\sgn}
    =
    (1+\alpha)\pi_0-\alpha\pi_0
    =
    \pi_0>0,
\)
and hence all source-initialized trajectories begin in \(\Omega_0^+\).

\paragraph{Derivation of the Signed RF velocity.}
For each branch, \(\pi_t^\pm(\bx)\) denotes the RF marginal density and
\(\pi_t^\pm(\bx)v_t^\pm(\bx)\) denotes the corresponding probability flux.
Both quantities depend linearly on the terminal distribution. Therefore, for
the signed target
\(
    \pi_1^{\sgn}
    =
    (1+\alpha)\pi_1^+
    -
    \alpha\pi_1^-,
\)
the induced signed marginal and flux are
\begin{align}
    \pi_t^{\sgn}(\bx)
    &=
    (1+\alpha)\,\pi_t^+(\bx)
    -
    \alpha\,\pi_t^-(\bx),
    \label{eq:signed_marginal_restate}
    \\
    \pi_t^{\sgn}(\bx)v_t^{\sgnrf}(\bx)
    &=
    (1+\alpha)\,\pi_t^+(\bx)v_t^+(\bx)
    -
    \alpha\,\pi_t^-(\bx)v_t^-(\bx).
    \label{eq:signed_flux_restate}
\end{align}
As in standard RF, the velocity field is obtained by dividing the flux by the
marginal density. Thus, wherever \(\pi_t^{\sgn}(\bx)\neq0\),
\begin{equation}
    v_t^{\sgnrf}(\bx)
    =
    \frac{
        (1+\alpha)\,\pi_t^+(\bx)v_t^+(\bx)
        -
        \alpha\,\pi_t^-(\bx)v_t^-(\bx)
    }{
        (1+\alpha)\,\pi_t^+(\bx)
        -
        \alpha\,\pi_t^-(\bx)
    }.
    \label{eq:signed_rf_velocity_restate}
\end{equation}
This recovers Eq.~\ref{eq:signed_rf_velocity} in the main text.

To obtain the guidance form, define
\(
    r_t(\bx)\coloneqq\pi_t^-(\bx)/\pi_t^+(\bx)
\)
and
\(
    \Delta v_t(\bx)\coloneqq v_t^+(\bx)-v_t^-(\bx)
\).
Dividing the numerator and denominator of
Eq.~\ref{eq:signed_rf_velocity_restate} by \(\pi_t^+(\bx)\) gives
\begin{align*}
    v_t^{\sgnrf}(\bx)
    &=
    \frac{
        (1+\alpha)v_t^+(\bx)
        -
        \alpha\,r_t(\bx)v_t^-(\bx)
    }{
        (1+\alpha)-\alpha\,r_t(\bx)
    }
    \notag
    \\
    &=
    v_t^+(\bx)
    +
    \frac{
        \alpha\,r_t(\bx)
    }{
        (1+\alpha)-\alpha\,r_t(\bx)
    }
    \Delta v_t(\bx).
\end{align*}
Hence,
\begin{equation}
    v_t^{\sgnrf}(\bx)
    =
    v_t^+(\bx)
    +
    \lambda_t^\alpha(\bx)\Delta v_t(\bx),
    \qquad
    \lambda_t^\alpha(\bx)
    =
    \frac{
        \alpha\,r_t(\bx)
    }{
        (1+\alpha)-\alpha\,r_t(\bx)
    }.
    \label{eq:signed_rf_guidance_restate}
\end{equation}
The denominator satisfies
\[
    (1+\alpha)-\alpha\,r_t(\bx)
    =
    \frac{\pi_t^{\sgn}(\bx)}{\pi_t^+(\bx)}.
\]
Therefore, assuming \(\pi_t^+(\bx)>0\), it vanishes precisely on the signed
zero set \(\Omega_t^0\), where the Signed RF velocity becomes singular.

\vspace{1.5em}

\begin{proposition}[TV-optimal probability approximation]
    \label{prop:tv_optimal_probability_approximation}
    Fix \(t<1\), and define the total negative mass of the signed marginal by
    \[
        M_t^-
        \coloneqq
        -
        \int_{\Omega_t^-}
        \pi_t^{\sgn}(\bx)\,\mathrm{d}\bx.
    \]
    Then, among all probability densities \(q\),
    \[
        \inf_{\substack{q(\bx)\geq 0\\
        \int q(\bx)\,\mathrm{d}\bx=1}}
        \frac{1}{2}
        \int_{\R^d}
        \left|
            q(\bx)-\pi_t^{\sgn}(\bx)
        \right|
        \mathrm{d}\bx
        =
        M_t^-.
    \]
    Moreover, \(q\) is optimal if and only if
    \[
        q(\bx)=0
        \quad\text{on }\Omega_t^0\cup\Omega_t^-,
        \qquad
        0\leq q(\bx)\leq\pi_t^{\sgn}(\bx)
        \quad\text{on }\Omega_t^+,
    \]
    up to sets of measure zero.

    In particular, if \(S_t\subseteq\Omega_t^+\) satisfies
    \[
        \int_{S_t}
        \pi_t^{\sgn}(\bx)\,\mathrm{d}\bx
        =
        1,
    \]
    then
    \[
        q_{S_t}(\bx)
        =
        \pi_t^{\sgn}(\bx)\mathbf{1}\{\bx\in S_t\}
    \]
    is TV-optimal. By Theorem~\ref{thm:signed_rf_sampling}, the Signed RF sampling
    law corresponds to \(S_t=\Omega_t^r\), and is therefore one such optimal
    approximation.
\end{proposition}

\begin{proof}
    Since \(\pi_t^{\sgn}\) has total mass one, its positive region contains
    total mass
    \[
        \int_{\Omega_t^+}
        \pi_t^{\sgn}(\bx)\,\mathrm{d}\bx
        =
        1+M_t^-.
    \]
    Now let \(q\) be any probability density. Define
    \[
        A
        \coloneqq
        \int_{\Omega_t^0\cup\Omega_t^-}
        q(\bx)\,\mathrm{d}\bx
    \]
    as the probability mass that \(q\) assigns outside the positive region,
    and define
    \[
        E
        \coloneqq
        \int_{\Omega_t^+}
        \bigl(q(\bx)-\pi_t^{\sgn}(\bx)\bigr)_+
        \,\mathrm{d}\bx
    \]
    as the amount by which \(q\) exceeds the signed density inside the positive
    region.

    Because \(q\) has mass \(1-A\) on \(\Omega_t^+\),
    \[
        \int_{\Omega_t^+}
        \bigl(
            \pi_t^{\sgn}(\bx)-q(\bx)
        \bigr)
        \,\mathrm{d}\bx
        =
        (1+M_t^-)-(1-A)
        =
        M_t^-+A.
    \]
    Therefore,
    \[
        \int_{\Omega_t^+}
        \left|
            q(\bx)-\pi_t^{\sgn}(\bx)
        \right|
        \mathrm{d}\bx
        =
        M_t^-+A+2E.
    \]
    On \(\Omega_t^0\cup\Omega_t^-\), we have
    \(\pi_t^{\sgn}(\bx)\leq0\) and \(q(\bx)\geq0\), so
    \[
        \int_{\Omega_t^0\cup\Omega_t^-}
        \left|
            q(\bx)-\pi_t^{\sgn}(\bx)
        \right|
        \mathrm{d}\bx
        =
        A+M_t^-.
    \]
    Combining the two regions gives
    \[
        \frac{1}{2}
        \int_{\R^d}
        \left|
            q(\bx)-\pi_t^{\sgn}(\bx)
        \right|
        \mathrm{d}\bx
        =
        M_t^-+A+E
        \geq
        M_t^-.
    \]
    Equality holds precisely when \(A=0\) and \(E=0\), meaning that \(q\)
    assigns no mass outside \(\Omega_t^+\) and never exceeds
    \(\pi_t^{\sgn}\) inside \(\Omega_t^+\). This proves the stated
    characterization.

    Finally, any
    \(q_{S_t}(\bx)=\pi_t^{\sgn}(\bx)\mathbf{1}\{\bx\in S_t\}\)
    with unit total mass satisfies these two conditions and is therefore
    TV-optimal.
\end{proof}

\vspace{1em}

\begin{assumption}[Regularity of the signed flow]
    \label{ass:signed_flow_regularity}
    Let \(\pi_1^+\) and \(\pi_1^-\) be two target distributions with finite
    first moments, and suppose that both RF branches use the Gaussian
    source \(\pi_0=\mathcal N(\boldsymbol 0,\mI)\). Let
    \((\pi_t^+,v_t^+)\) and \((\pi_t^-,v_t^-)\) denote the corresponding RF
    marginals and velocity fields.

    We assume that the signed marginal
    \[
        (t,\bx)
        \longmapsto
        \pi_t^{\sgn}(\bx)
        =
        (1+\alpha)\pi_t^+(\bx)-\alpha\pi_t^-(\bx)
    \]
    and the signed probability flux
    \[
        (t,\bx)
        \longmapsto
        (1+\alpha)\pi_t^+(\bx)v_t^+(\bx)
        -
        \alpha\pi_t^-(\bx)v_t^-(\bx)
    \]
    are locally \(C^1\) on \((0,1)\times\R^d\) and, together with their
    spatial gradients, extend continuously to \([0,1)\times\R^d\). Let
    \[
        \mathcal D
        \coloneqq
        \bigl\{
            (t,\bx):
            0<t<1,\
            \pi_t^{\sgn}(\bx)=0,\
            \nabla\pi_t^{\sgn}(\bx)=0
        \bigr\}
    \]
    denote the degenerate points of the zero set. We further assume:
    \begin{enumerate}[label=(\roman*),leftmargin=1.8em]
        \item \emph{Avoidance of degenerate zeros.} For \(\pi_0\)-almost every
        \(\bz_0\), the Signed RF trajectory started at \(\bz_0\) has no limit
        point in \(\mathcal D\) at the end of its maximal interval of existence.
        \item \emph{Non-explosion.} For \(\pi_0\)-almost every \(\bz_0\), the
        Signed RF trajectory started at \(\bz_0\) is bounded on \(I\cap[0,T]\)
        for each \(T<1\), where \(I\) is its maximal interval of existence.
    \end{enumerate}
\end{assumption}

For the common Gaussian source \(\pi_0=\mathcal N(\boldsymbol 0,\mI)\),
Tweedie's identity gives, for each RF branch,
\[
    \pi_t^\pm(\bx)v_t^\pm(\bx)
    =
    \frac{\bx}{t}\,\pi_t^\pm(\bx)
    +
    \frac{1-t}{t}\,\nabla\pi_t^\pm(\bx),
    \qquad
    0<t<1.
\]
Taking the same signed combination as in the definition of
\(\pi_t^{\sgn}\) yields
\[
    (1+\alpha)\pi_t^+(\bx)v_t^+(\bx)
    -
    \alpha\pi_t^-(\bx)v_t^-(\bx)
    =
    \frac{\bx}{t}\,\pi_t^{\sgn}(\bx)
    +
    \frac{1-t}{t}\,\nabla\pi_t^{\sgn}(\bx).
\]
Therefore, at any nondegenerate point of the zero set, where
\(\pi_t^{\sgn}(\bx)=0\) and
\(\nabla\pi_t^{\sgn}(\bx)\neq0\),
\[
    \bigl(\nabla\pi_t^{\sgn}(\bx)\bigr)^\top
    \left[
        (1+\alpha)\pi_t^+(\bx)v_t^+(\bx)
        -
        \alpha\pi_t^-(\bx)v_t^-(\bx)
    \right]
    =
    \frac{1-t}{t}
    \left\|\nabla\pi_t^{\sgn}(\bx)\right\|^2
    >
    0.
\]
Thus, the signed flux points locally toward increasing
\(\pi_t^{\sgn}\), i.e., toward the positive side of the zero set. This is the
barrier property used in the nonpenetration proof below.

\begin{proposition}[Signed RF remains in the positive region]
    \label{prop:signed_rf_nonpenetration}
    Suppose that Assumption~\ref{ass:signed_flow_regularity} holds. Then, for
    \(\pi_0\)-almost every \(\bz_0\), the Signed RF trajectory started at
    \(\bz_0\) exists on \([0,1)\) and satisfies \(\bz_t\in\Omega_t^+\) for all
    \(t\in[0,1)\). Consequently, the law \(\pi_t^{\sgnrf}\) of the
    source-initialized solution satisfies
    \[
        \pi_t^{\sgnrf}(\Omega_t^+)
        =
        1,
        \qquad
        \text{for every }t\in[0,1).
    \]
\end{proposition}

\begin{proof}
    Let \(\bz_t\) be a source-initialized Signed RF trajectory. Since both RF
    branches share the same Gaussian source,
    \[
        \pi_0^{\sgn}(\bz_0)
        =
        (1+\alpha)\pi_0(\bz_0)-\alpha\pi_0(\bz_0)
        =
        \pi_0(\bz_0)
        >
        0.
    \]
    Thus, every trajectory starts in the positive region \(\Omega_0^+\).

    Define the signed probability current
    \[
        j_t^{\sgn}(\bx)
        \coloneqq
        (1+\alpha)\pi_t^+(\bx)v_t^+(\bx)
        -
        \alpha\pi_t^-(\bx)v_t^-(\bx).
    \]
    By linearity of the two branch continuity equations,
    \[
        \partial_t\pi_t^{\sgn}
        +
        \nabla\!\cdot j_t^{\sgn}
        =
        0.
    \]
    Moreover, wherever \(\pi_t^{\sgn}(\bx)\neq0\),
    \[
        v_t^{\sgnrf}(\bx)
        =
        \frac{
            j_t^{\sgn}(\bx)
        }{
            \pi_t^{\sgn}(\bx)
        }.
    \]
    Therefore, as long as the trajectory remains in \(\Omega_t^+\), the chain
    rule gives
    \begin{align}
        \frac{\mathrm d}{\mathrm dt}
        \pi_t^{\sgn}(\bz_t)
        &=
        \partial_t\pi_t^{\sgn}(\bz_t)
        +
        \bigl(\nabla\pi_t^{\sgn}(\bz_t)\bigr)^\top\dot{\bz}_t
        \notag
        \\
        &=
        -\nabla\!\cdot j_t^{\sgn}(\bz_t)
        +
        \frac{
            \bigl(\nabla\pi_t^{\sgn}(\bz_t)\bigr)^\top
            j_t^{\sgn}(\bz_t)
        }{
            \pi_t^{\sgn}(\bz_t)
        }.
        \label{eq:signed_density_along_trajectory}
    \end{align}

    Since \(v_t^{\sgnrf}=j_t^{\sgn}/\pi_t^{\sgn}\) and its spatial Jacobian
    are continuous on the relatively open set
    \(
        U\coloneqq\{(t,\bx):0\le t<1,\ \pi_t^{\sgn}(\bx)>0\},
    \)
    the Signed RF ODE has a unique maximal solution \(\bz_t\) in \(U\), defined on some
    interval \([0,t_{\max})\) with \(0<t_{\max}\le1\); if \(t_{\max}<1\),
    then \((t,\bz_t)\) eventually leaves every compact subset of \(U\) as
    \(t\uparrow t_{\max}\). Fix \(\bz_0\) in the full-measure set on which
    Assumption~\ref{ass:signed_flow_regularity}(i) and (ii) hold, and suppose,
    for contradiction, that \(t_{\max}<1\).

    By Assumption~\ref{ass:signed_flow_regularity}(ii), \(\bz_t\) stays in
    a closed ball \(\bar B_R\) for \(t<t_{\max}\). Hence the set \(K\) of
    limit points of \((t,\bz_t)\) as \(t\uparrow t_{\max}\) is a nonempty
    compact subset of \(\{t_{\max}\}\times\bar B_R\). By the escape property,
    \(K\cap U=\emptyset\); since \(\pi_t^{\sgn}(\bz_t)>0\), continuity gives
    \(\pi_{t_{\max}}^{\sgn}(\bx)=0\) for every \((t_{\max},\bx)\in K\), and
    \(\pi_t^{\sgn}(\bz_t)\to0\) as \(t\uparrow t_{\max}\).

    By Assumption~\ref{ass:signed_flow_regularity}(i),
    \(K\cap\mathcal D=\emptyset\), so
    \(\nabla\pi_{t_{\max}}^{\sgn}(\bx)\neq0\) on \(K\). By the
    Gaussian-source identity derived above,
    \[
        \bigl(\nabla\pi_{t_{\max}}^{\sgn}(\bx)\bigr)^\top
        j_{t_{\max}}^{\sgn}(\bx)
        =
        \frac{1-t_{\max}}{t_{\max}}
        \left\|\nabla\pi_{t_{\max}}^{\sgn}(\bx)\right\|^2
        >
        0
        \qquad
        \text{for }(t_{\max},\bx)\in K.
    \]
    By continuity and compactness of \(K\), there exist an open neighborhood
    \(N\) of \(K\) and constants \(c>0\) and \(M<\infty\) such that
    \[
        \bigl(\nabla\pi_t^{\sgn}(\bx)\bigr)^\top j_t^{\sgn}(\bx)
        \geq
        c,
        \qquad
        \left|\nabla\!\cdot j_t^{\sgn}(\bx)\right|
        \leq
        M,
        \qquad
        (t,\bx)\in N,
    \]
    and some \(\delta>0\) with \((t,\bz_t)\in N\) for
    \(t\in(t_{\max}-\delta,t_{\max})\). Equation~\ref{eq:signed_density_along_trajectory} therefore
    implies
    \[
        \frac{\mathrm d}{\mathrm dt}
        \pi_t^{\sgn}(\bz_t)
        \geq
        -M
        +
        \frac{c}{\pi_t^{\sgn}(\bz_t)},
        \qquad
        t\in(t_{\max}-\delta,t_{\max}).
    \]
    As \(\pi_t^{\sgn}(\bz_t)\to0\), the right-hand side becomes positive, so
    the positive function \(t\mapsto\pi_t^{\sgn}(\bz_t)\) is increasing near
    \(t_{\max}\) and cannot tend to \(0\), a contradiction.

    Hence \(t_{\max}=1\) for \(\pi_0\)-almost every \(\bz_0\), and
    \(\pi_t^{\sgnrf}(\Omega_t^+)=1\) for every \(t\in[0,1)\).
\end{proof}

\begin{theorem}[Sampling law of Signed RF]
    \label{thm:signed_rf_sampling}
    Suppose that Assumption~\ref{ass:signed_flow_regularity} holds. For
    \(t<1\), let \(G_t\) be the set of initial points \(\bz_0\) whose Signed RF
    trajectory exists on \([0,t]\), let \(\Phi_t(\bz_0)\) denote its position at
    time \(t\), and define the reachable region by
    \[
        \Omega_t^r
        \coloneqq
        \bigl\{
            \Phi_t(\bz_0):
            \bz_0\in G_t
        \bigr\}.
    \]
    Then, for every \(t<1\), the density of the source-initialized Signed RF
    samples is
    \[
        \pi_t^{\sgnrf}(\bx)
        =
        \pi_t^{\sgn}(\bx)\,
        \mathbf{1}\{\bx\in\Omega_t^r\}.
    \]
    In particular,
    \[
        \int_{\Omega_t^r}
        \pi_t^{\sgn}(\bx)\,\mathrm{d}\bx
        =
        1.
    \]
\end{theorem}

\begin{proof}
    Recall that the signed probability current is
    \[
        j_t^{\sgn}(\bx)
        =
        (1+\alpha)\pi_t^+(\bx)v_t^+(\bx)
        -
        \alpha\pi_t^-(\bx)v_t^-(\bx).
    \]
    By linearity of the two branch continuity equations,
    \[
        \partial_t\pi_t^{\sgn}
        +
        \nabla\!\cdot j_t^{\sgn}
        =
        0.
    \]
    Moreover, wherever \(\pi_t^{\sgn}(\bx)\neq0\),
    \[
        j_t^{\sgn}(\bx)
        =
        \pi_t^{\sgn}(\bx)v_t^{\sgnrf}(\bx).
    \]
    Therefore, on the positive region,
    \[
        \partial_t\pi_t^{\sgn}
        +
        \nabla\!\cdot
        \bigl(
            \pi_t^{\sgn}v_t^{\sgnrf}
        \bigr)
        =
        0.
    \]

    By Proposition~\ref{prop:signed_rf_nonpenetration}, \(\pi_0(G_t)=1\) for
    every \(t<1\). Since the Signed RF velocity and its spatial Jacobian are
    continuous on the relatively open set where \(\pi_t^{\sgn}>0\), standard ODE theory implies that \(G_t\) is
    open and that \(\Phi_t\) is a \(C^1\) diffeomorphism from \(G_t\) onto
    \(\Omega_t^r\).

    Let
    \[
        J_t(\bz_0)
        \coloneqq
        \left|
            \det D\Phi_t(\bz_0)
        \right|
    \]
    denote the Jacobian determinant of the flow map. Along the trajectory
    \(\Phi_t(\bz_0)\), Liouville's formula gives
    \[
        \frac{\mathrm d}{\mathrm dt}
        J_t(\bz_0)
        =
        J_t(\bz_0)\,
        \nabla\!\cdot
        v_t^{\sgnrf}\bigl(\Phi_t(\bz_0)\bigr).
    \]
    Combining this identity with the continuity equation yields
    \begin{align}
        \frac{\mathrm d}{\mathrm dt}
        \Bigl[
            \pi_t^{\sgn}\bigl(\Phi_t(\bz_0)\bigr)
            J_t(\bz_0)
        \Bigr]
        &=
        \Bigl[
            \partial_t\pi_t^{\sgn}
            +
            \bigl(\nabla\pi_t^{\sgn}\bigr)^\top v_t^{\sgnrf}
            +
            \pi_t^{\sgn}\nabla\!\cdot v_t^{\sgnrf}
        \Bigr]_{\bx=\Phi_t(\bz_0)}
        J_t(\bz_0)
        \notag
        \\
        &=
        \Bigl[
            \partial_t\pi_t^{\sgn}
            +
            \nabla\!\cdot
            \bigl(
                \pi_t^{\sgn}v_t^{\sgnrf}
            \bigr)
        \Bigr]_{\bx=\Phi_t(\bz_0)}
        J_t(\bz_0)
        \notag
        \\
        &=
        0.
        \label{eq:signed_mass_conservation_along_flow}
    \end{align}
    Thus, signed mass is conserved along the flow map. Since
    \(\Phi_0(\bz_0)=\bz_0\), \(J_0(\bz_0)=1\), and
    \[
        \pi_0^{\sgn}
        =
        (1+\alpha)\pi_0-\alpha\pi_0
        =
        \pi_0,
    \]
    Eq.~\ref{eq:signed_mass_conservation_along_flow} gives
    \[
        \pi_t^{\sgn}\bigl(\Phi_t(\bz_0)\bigr)
        J_t(\bz_0)
        =
        \pi_0(\bz_0).
    \]

    Now let \(f\) be any compactly supported test function. By the definition
    of the source-initialized sampling law,
    \[
        \int_{\R^d}
        f(\bx)\pi_t^{\sgnrf}(\bx)\,\mathrm{d}\bx
        =
        \int_{G_t}
        f\bigl(\Phi_t(\bz_0)\bigr)
        \pi_0(\bz_0)\,\mathrm{d}\bz_0.
    \]
    Substituting the conservation identity above and applying the
    change-of-variables formula gives
    \begin{align*}
        \int_{\R^d}
        f(\bx)\pi_t^{\sgnrf}(\bx)\,\mathrm{d}\bx
        &=
        \int_{G_t}
        f\bigl(\Phi_t(\bz_0)\bigr)
        \pi_t^{\sgn}\bigl(\Phi_t(\bz_0)\bigr)
        J_t(\bz_0)\,\mathrm{d}\bz_0
        \\
        &=
        \int_{\Omega_t^r}
        f(\bx)\pi_t^{\sgn}(\bx)\,\mathrm{d}\bx.
    \end{align*}
    Since this holds for every test function \(f\),
    \[
        \pi_t^{\sgnrf}(\bx)
        =
        \pi_t^{\sgn}(\bx)\,
        \mathbf{1}\{\bx\in\Omega_t^r\}
    \]
    almost everywhere. Both sides are integrable; integrating this identity
    over \(\mathbb R^d\) (equivalently, applying it to compactly supported
    cutoffs converging to \(1\) and using dominated convergence) gives
    \[
        \int_{\Omega_t^r}
        \pi_t^{\sgn}(\bx)\,\mathrm{d}\bx
        =
        1.
    \]
\end{proof}

\begin{remark}[Discussion of the assumptions]
    The regularity in Assumption~\ref{ass:signed_flow_regularity} ensures that the continuity equation,
    the Signed RF ODE, and the Jacobian calculation above are well defined.
    Such assumptions are standard in continuous-time flow
    analyses~\citep{chen2018neural,lipman2022flow,albergo2022building,
    albergo2023stochastic}.

    Away from \(\mathcal D\), the zero set is locally a smooth boundary, and
    the Gaussian-source identity makes the signed current point toward its
    positive side. Degenerate zeros cannot be excluded altogether: where a
    negative component is born or two components merge, \(\pi_t^{\sgn}\) has a
    critical point with value \(0\). For example, for \(d=1\),
    \(\pi_1^+=\mathcal N(0,1)\), \(\pi_1^-=\mathcal N(0,0.2^2)\), and
    \(\alpha=1\), a negative component is born at \(\bx_\ast=0\) at
    \(t_\ast\approx0.654\), where \(\nabla\pi_{t_\ast}^{\sgn}(\bx_\ast)=0\).
    By symmetry, \(\bz_t\equiv0\) is a trajectory and lies in \(\Omega_t^-\)
    after \(t_\ast\). For \(t\) close to \(t_\ast\), the velocity near \(0\)
    points away from \(0\), and one-dimensional trajectories cannot cross
    \(\bz_t\equiv0\), so every other trajectory avoids \((t_\ast,\bx_\ast)\)
    and condition (i) holds. Condition (ii) rules out finite-time blow-up, the
    only other way a solution can cease to exist before \(t=1\).

    We state the theorem for \(t<1\) because the Gaussian interpolation smooths
    the intermediate marginals whenever \(t<1\), whereas the terminal
    distributions at \(t=1\) may be nonsmooth or singular.
\end{remark}

\vspace{1em}

\paragraph{Online density-ratio tracking.}
We derive the evolution of the density ratio along a source-initialized Signed
RF trajectory \(\bz_t\). Recall
\[
    r_t(\bx)
    =
    \frac{\pi_t^-(\bx)}{\pi_t^+(\bx)},
    \qquad
    u_t
    \coloneqq
    \log r_t(\bz_t).
\]
Since the two branches share the same source distribution,
\(r_0(\bz_0)=1\), and therefore \(u_0=0\).

Let
\(
    s_t^\pm(\bx)\coloneqq\nabla\log\pi_t^\pm(\bx)
\)
denote the branch score functions. Each branch satisfies the continuity
equation
\[
    \partial_t\pi_t^\pm
    +
    \nabla\!\cdot
    \bigl(
        \pi_t^\pm v_t^\pm
    \bigr)
    =
    0.
\]
Dividing by \(\pi_t^\pm\) gives
\[
    \partial_t\log\pi_t^\pm
    =
    -\nabla\!\cdot v_t^\pm
    -
    (v_t^\pm)^\top s_t^\pm.
\]

Along the Signed RF trajectory,
\[
    \dot{\bz}_t
    =
    v_t^+(\bz_t)
    +
    \lambda_t^\alpha(\bz_t)\Delta v_t(\bz_t),
    \qquad
    \Delta v_t
    \coloneqq
    v_t^+-v_t^-,
\]
where
\[
    \lambda_t^\alpha(\bz_t)
    =
    \frac{
        \alpha\exp(u_t)
    }{
        (1+\alpha)-\alpha\exp(u_t)
    }.
\]
Using
\(
    u_t
    =
    \log\pi_t^-(\bz_t)-\log\pi_t^+(\bz_t)
\)
and applying the chain rule yields
\begin{align}
    \dot{u}_t
    &=
    \nabla\!\cdot\Delta v_t(\bz_t)
    +
    \Delta v_t(\bz_t)^\top s_t^-(\bz_t)
    \notag
    \\
    &\quad
    +
    \lambda_t^\alpha(\bz_t)\,
    \Delta v_t(\bz_t)^\top
    \bigl(
        s_t^-(\bz_t)-s_t^+(\bz_t)
    \bigr).
    \label{eq:online_ratio_tracking_restate}
\end{align}
This recovers Eq.~\ref{eq:online_ratio_tracking} in the main text.

For RF with the Gaussian source
\(\pi_0=\mathcal N(\boldsymbol 0,\mI)\), the branch scores satisfy
\[
    s_t^\pm(\bx)
    =
    \frac{
        t\,v_t^\pm(\bx)-\bx
    }{
        1-t
    },
    \qquad
    s_t^-(\bx)-s_t^+(\bx)
    =
    -\frac{t}{1-t}\Delta v_t(\bx).
\]
Substituting these identities into
Eq.~\ref{eq:online_ratio_tracking_restate} gives the velocity-only update
\begin{equation}
    \dot{u}_t
    =
    \nabla\!\cdot\Delta v_t(\bz_t)
    +
    \frac{
        \Delta v_t(\bz_t)^\top
        \bigl(
            t\,v_t^-(\bz_t)-\bz_t
        \bigr)
    }{
        1-t
    }
    -
    \lambda_t^\alpha(\bz_t)
    \frac{t}{1-t}
    \bigl\|
        \Delta v_t(\bz_t)
    \bigr\|^2.
    \label{eq:online_ratio_tracking_velocity_only}
\end{equation}

Thus, the density ratio can be tracked jointly with the Signed RF trajectory
by solving the augmented ODE
\[
\begin{cases}
    \displaystyle
    \dot{\bz}_t
    =
    v_t^+(\bz_t)
    +
    \lambda_t^\alpha(\bz_t)\Delta v_t(\bz_t),
    \\[0.8em]
    \displaystyle
    \dot{u}_t
    =
    \nabla\!\cdot\Delta v_t(\bz_t)
    +
    \dfrac{
        \Delta v_t(\bz_t)^\top
        \bigl(
            t\,v_t^-(\bz_t)-\bz_t
        \bigr)
    }{
        1-t
    }
    -
    \lambda_t^\alpha(\bz_t)
    \dfrac{t}{1-t}
    \bigl\|
        \Delta v_t(\bz_t)
    \bigr\|^2,
\end{cases}
\]
with
\[
    \bz_0\sim\pi_0,
    \qquad
    u_0=0,
    \qquad
    \lambda_t^\alpha(\bz_t)
    =
    \frac{
        \alpha\exp(u_t)
    }{
        (1+\alpha)-\alpha\exp(u_t)
    }.
\]
The density ratio along the trajectory is then recovered as
\(
    r_t(\bz_t)=\exp(u_t)
\).
The divergence
\(\nabla\!\cdot\Delta v_t(\bz_t)\) can be evaluated exactly when tractable or
estimated using Hutchinson's trace estimator.

\section{Additional Experimental Details}

\subsection{Additional 2D Gaussian Mixture Experiments}
\label{sec:additional_gaussian_mixtures}

\begin{figure*}
    \centering
    \setlength{\abovecaptionskip}{2pt}
    \setlength{\belowcaptionskip}{0pt}
    \vspace{-1.5em}
    \includegraphics[width=0.92\textwidth]{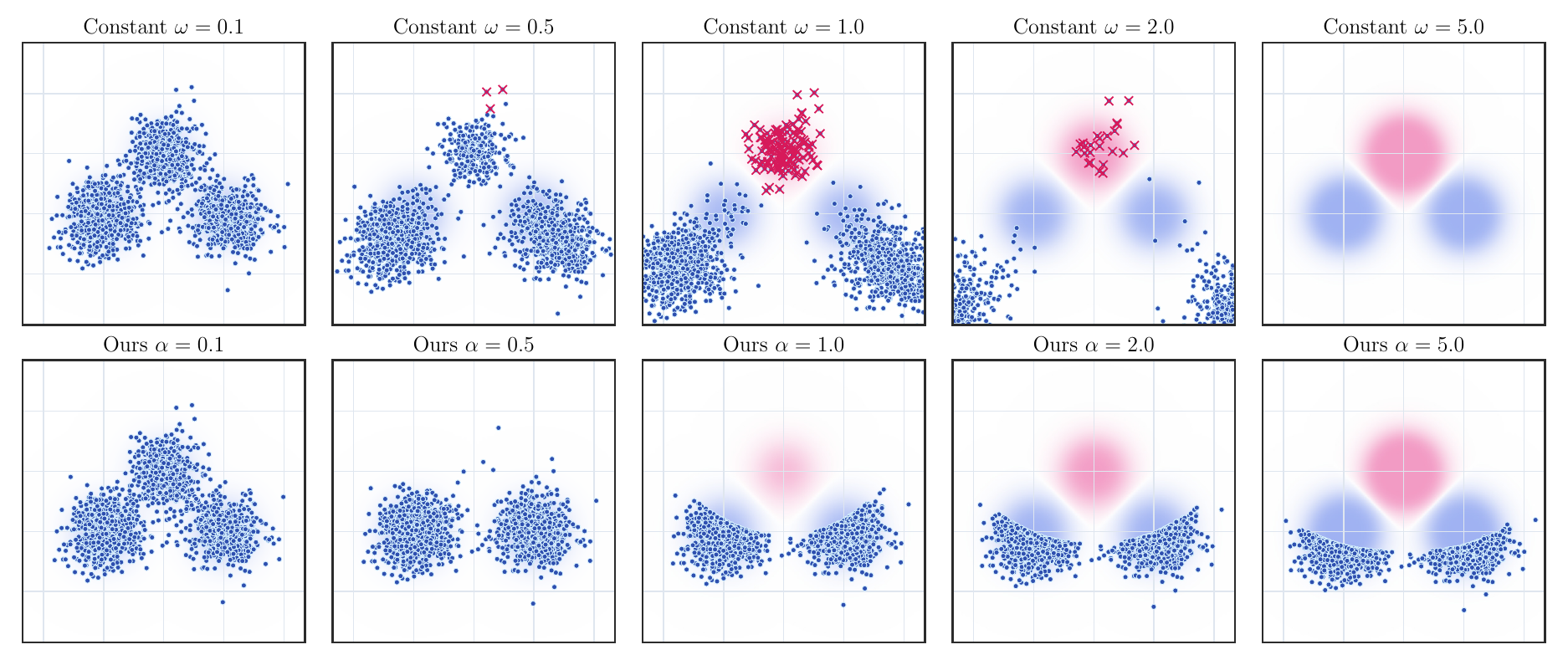}

    \includegraphics[width=0.92\textwidth]{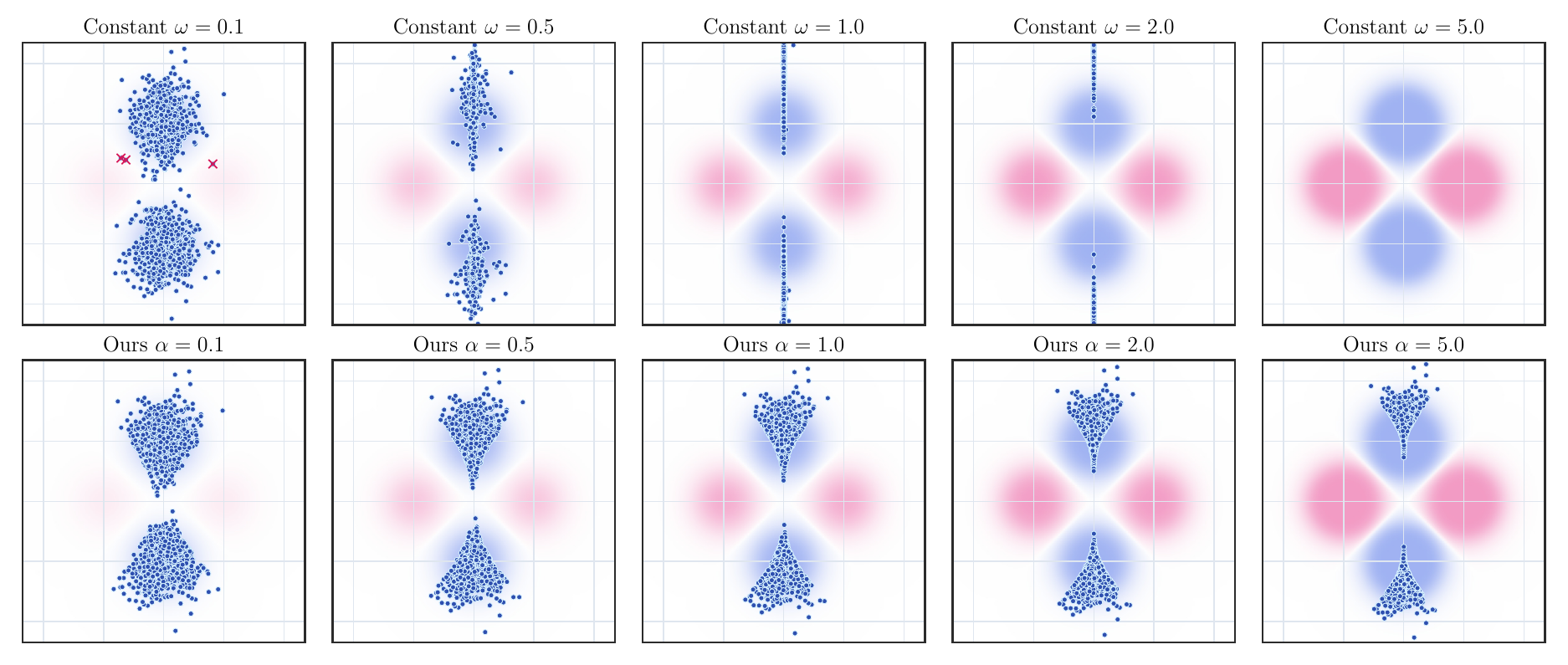}

    \includegraphics[width=0.92\textwidth]{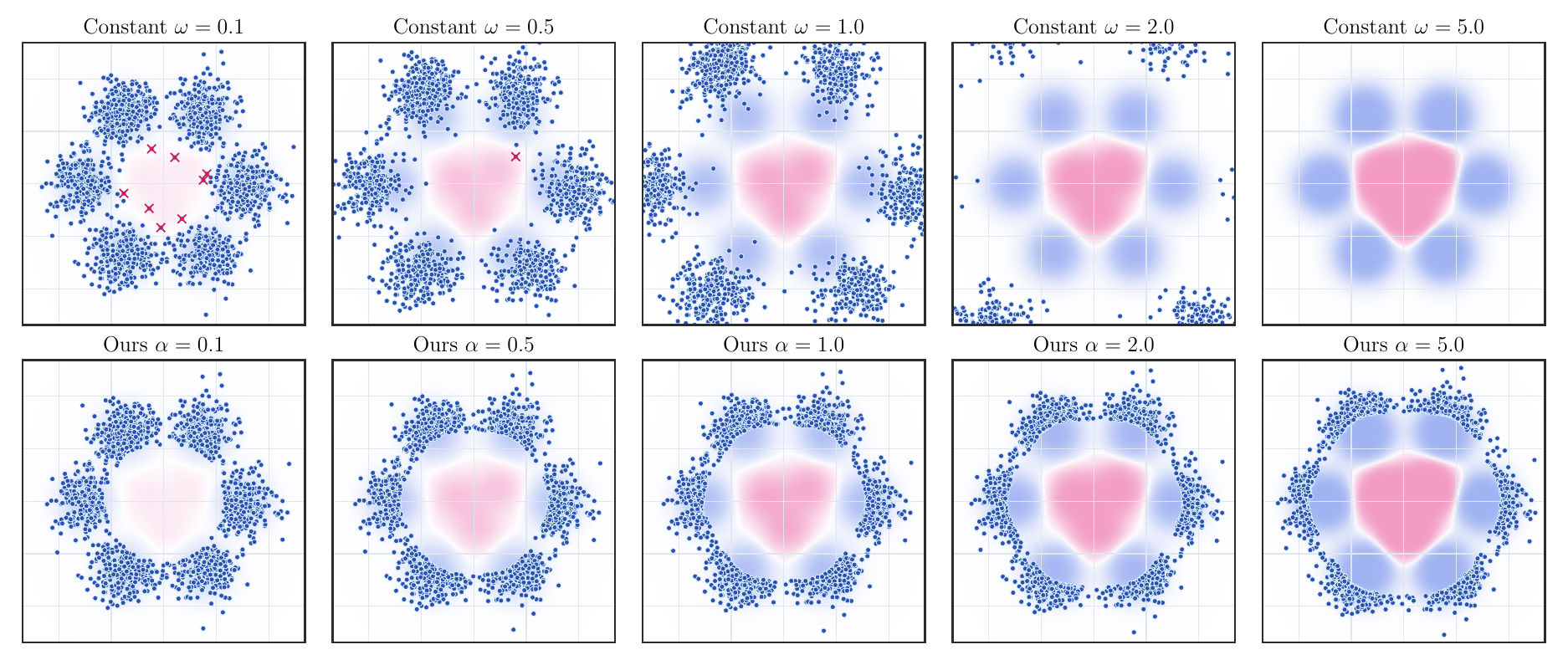}

    \includegraphics[width=0.92\textwidth]{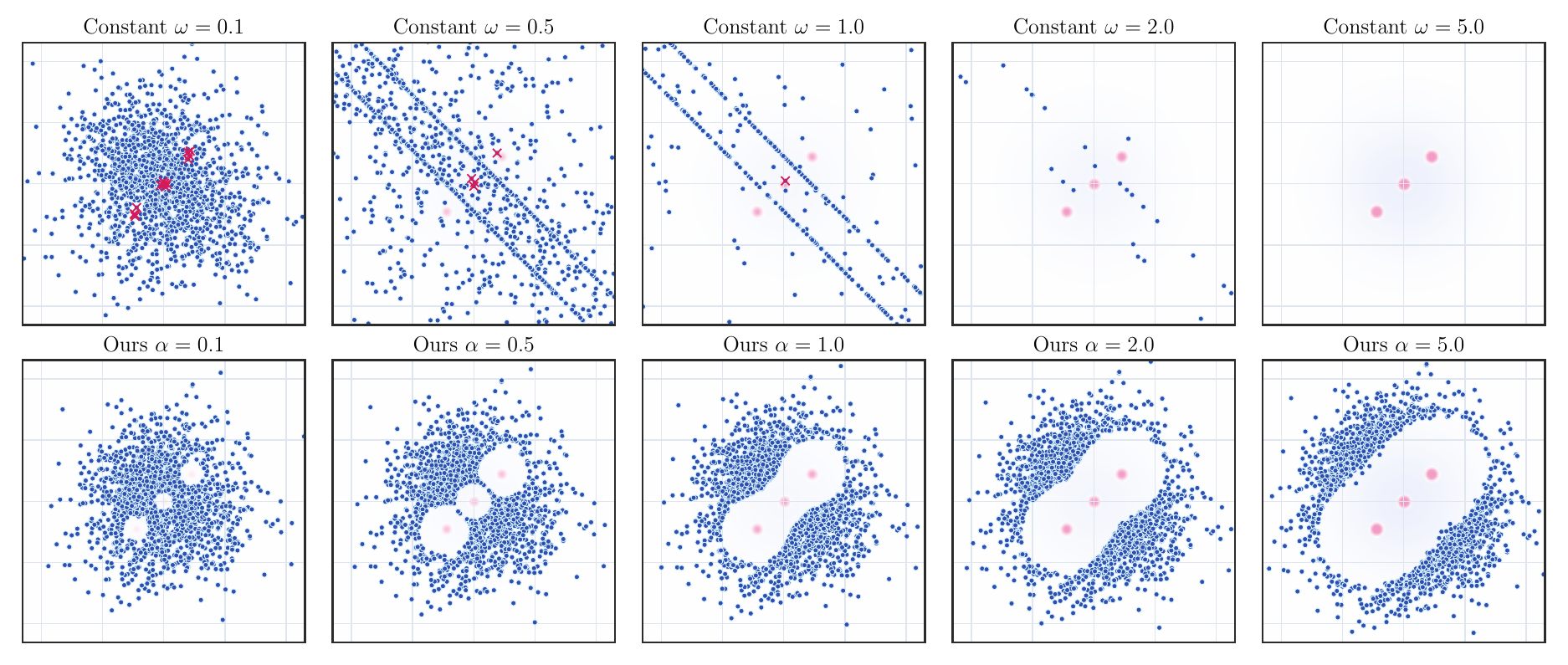}

    \caption{
        Additional 2D examples across multiple setups.}
    \label{fig:additional_gaussian_mixtures}
\end{figure*}

\begin{figure}
    \centering
    \includegraphics[width=\textwidth]{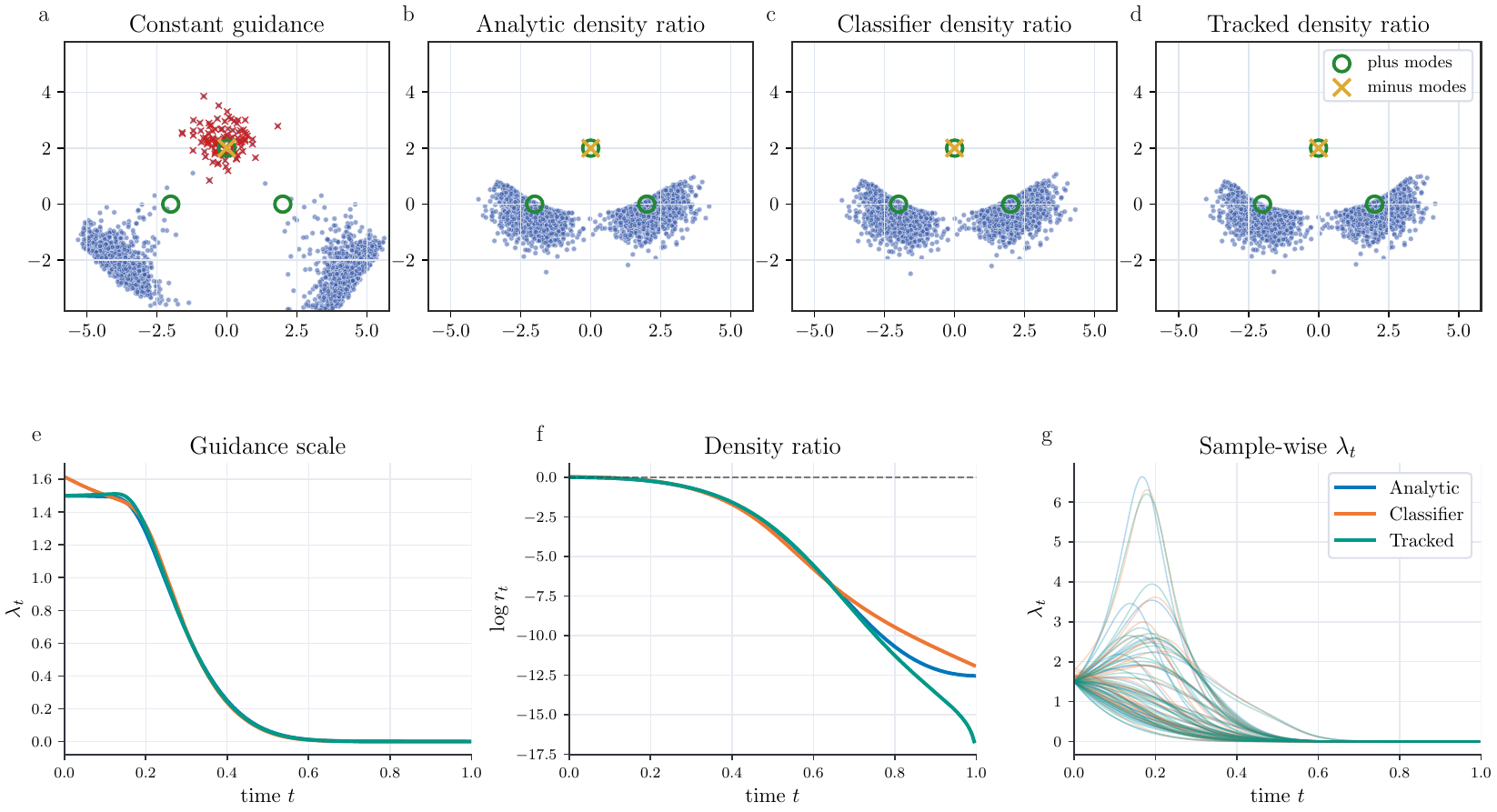}
    \caption{\textbf{Density-ratio estimation ablation.}
    Columns compare constant guidance, the analytic ratio, the classifier-based
    estimate, and the online-tracked estimate, using the same initial Gaussian
    noise and velocity fields \(v_t^+\) and \(v_t^-\). Panels (e)--(g) show the
    mean guidance scale \(\lambda_t^\alpha\), the mean log-density ratio \(\log r_t\), and
    the sample-wise guidance scales along individual trajectories. Despite small
    estimation errors, both practical estimators produce samples close to those
    obtained with the analytic ratio, indicating robustness to moderate
    density-ratio estimation error.}
    \label{fig:density_ratio_ablation_bimodal}
\end{figure}

\paragraph{Additional results on 2D toys.} Each toy dataset consists of a positive target \(\pi_1^+\) and a negative target \(\pi_1^-\), both defined as mixtures of isotropic Gaussians in \(\mathbb R^2\) with specified component means, weights, and variances. We train separate MLP-based RF models on samples from the two targets, yielding the velocity fields \(v_t^+\) and \(v_t^-\). Each velocity model is a four-layer MLP with hidden width \(256\), SiLU activations, and a linear two-dimensional output layer. We train the models for \(10{,}000\) steps using Adam~\citep{kingma2014adam}, with a learning rate of \(10^{-3}\) and batch size
\(1024\).

For Gaussian-mixture targets, the intermediate RF marginals \(\pi_t^+\) and \(\pi_t^-\) remain Gaussian mixtures, allowing the density ratio to be evaluated in closed form. In particular, consider a Gaussian component
\(\bX_1\sim\mathcal N(\boldsymbol{\mu}_k,\sigma_k^2\mI)\) and the source \(\bX_0\sim\mathcal N(\boldsymbol 0,\mI)\). Under the RF interpolation \(\bX_t=t\bX_1+(1-t)\bX_0\), its intermediate marginal is
\(
    \bX_t
    \sim
    \mathcal N\!\left(
        t\boldsymbol{\mu}_k,\,
        \bigl((1-t)^2+t^2\sigma_k^2\bigr)\mI
    \right).
\)
Therefore,
\begin{equation}
    r_t(\bx)
    =
    \frac{\pi_t^-(\bx)}{\pi_t^+(\bx)}
    =
    \frac{
        \displaystyle
        \sum_{k=1}^{K_-}
        w_k^-\,
        \mathcal N\!\left(
            \bx;\,
            t\boldsymbol{\mu}_k^-,\,
            \bigl((1-t)^2+t^2(\sigma_k^-)^2\bigr)\mI
        \right)
    }{
        \displaystyle
        \sum_{k=1}^{K_+}
        w_k^+\,
        \mathcal N\!\left(
            \bx;\,
            t\boldsymbol{\mu}_k^+,\,
            \bigl((1-t)^2+t^2(\sigma_k^+)^2\bigr)\mI
        \right)
    }.
    \label{eq:gaussian_mixture_density_ratio}
\end{equation}
We use this analytic ratio as ground truth, allowing us to study the behavior
of Signed RF independently of density-ratio estimation error.
Fig.~\ref{fig:additional_gaussian_mixtures} shows additional visualizations across a range of geometries. In all cases, sampling starts from the same Gaussian source and uses Euler integration with \(200\) steps. The background heatmap visualizes the target signed density at \(t=1\), \(\pi_1^{\sgn}=(1+\alpha)\,\pi_1^+-\alpha\,\pi_1^-\); blue denotes positive regions and pink denotes negative regions.

\paragraph{Density-ratio estimation analysis.}
We evaluate the accuracy and stability of the density-ratio estimators. In the same 2D setting, we train a density-ratio classifier with the same depth and hidden width as the velocity networks. The classifier outputs a single logit and is trained with binary cross-entropy on balanced samples from the positive and negative branches. We use Adam for \(10{,}000\) steps, with a learning rate of \(10^{-3}\) and batch size \(2048\).

Using the closed-form ratio in Eq.~\ref{eq:gaussian_mixture_density_ratio} as ground truth, we compare three variants of Signed RF: the analytic ratio, the classifier-based estimate, and the online-tracked estimate. All variants use the same initial Gaussian noise and the same velocity fields \(v_t^+\) and \(v_t^-\), differing only in how \(r_t(\bx)\) is obtained during sampling. Results are shown in Figs.~\ref{fig:density_ratio_ablation_bimodal}, \ref{fig:density_ratio_ablation_vertical}, \ref{fig:density_ratio_ablation_ring}, and \ref{fig:density_ratio_ablation_diagonal}. Panel (e) shows the mean effective guidance scale \(\lambda_t^\alpha\); the three variants agree closely, except near the beginning of the trajectory. Panel (f) shows the mean \(\log r_t\). Both practical estimators remain close to the analytic ratio over most of the trajectory, although the tracked estimate deviates slightly near the end, likely because of accumulated numerical integration error. Panel (g) shows the sample-wise guidance scales, which vary substantially across trajectories and toy configurations. This confirms that Signed RF guidance is state- and sample-dependent rather than a fixed time schedule.

Despite the observed differences in \(r_t\) and \(\lambda_t^\alpha\), the analytic,
classifier-based, and tracked variants produce visually similar samples. These
results suggest that Signed RF is robust to moderate density-ratio estimation
error in these examples, with both practical estimators closely approaching
the sampling behavior obtained using the exact ratio.

\begin{figure}
    \centering
    \includegraphics[width=\textwidth]{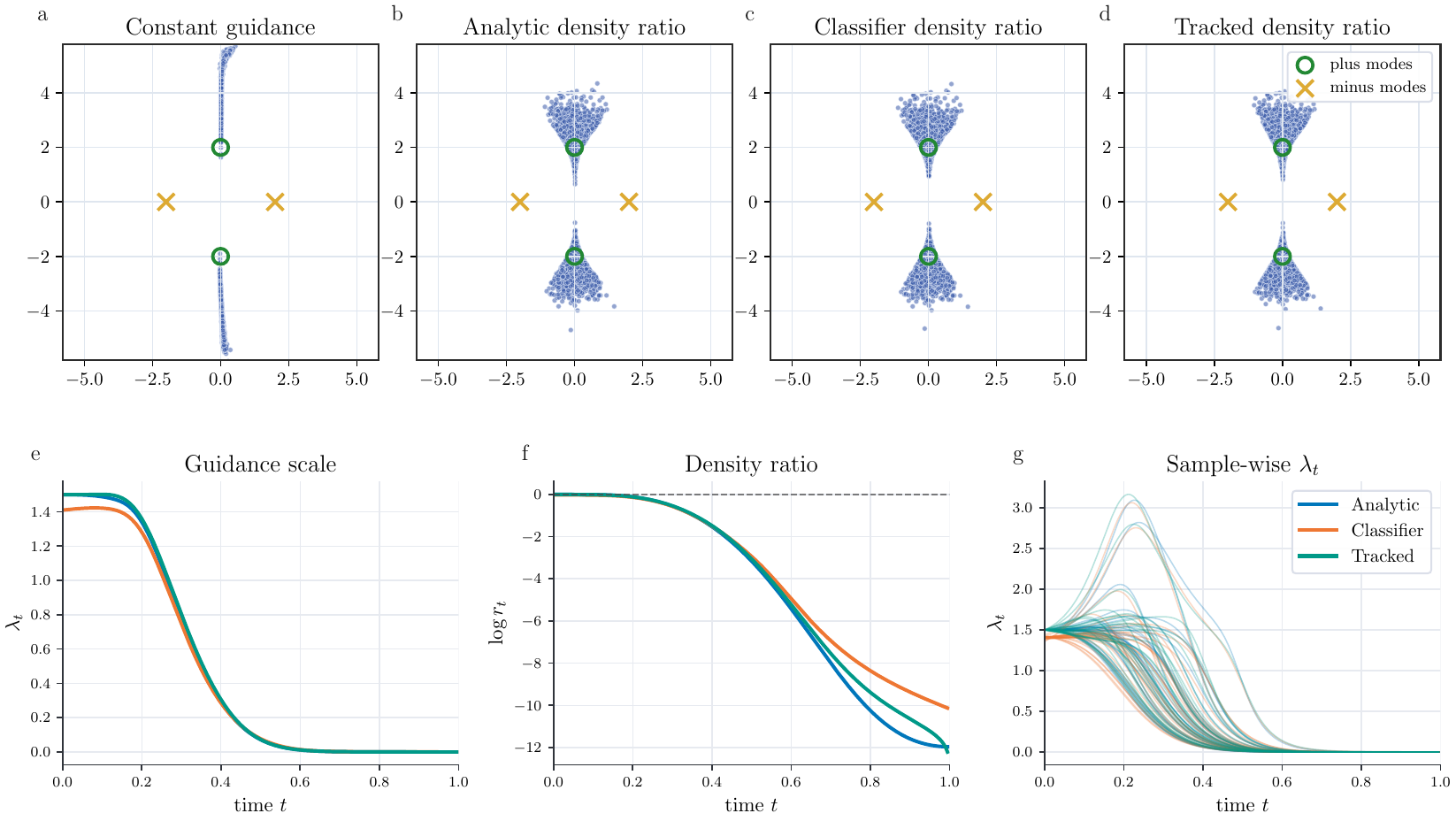}
    \caption{Additional density-ratio estimation ablation in a different setup.}
    \label{fig:density_ratio_ablation_vertical}
\end{figure}

\begin{figure}
    \centering
    \includegraphics[width=\textwidth]{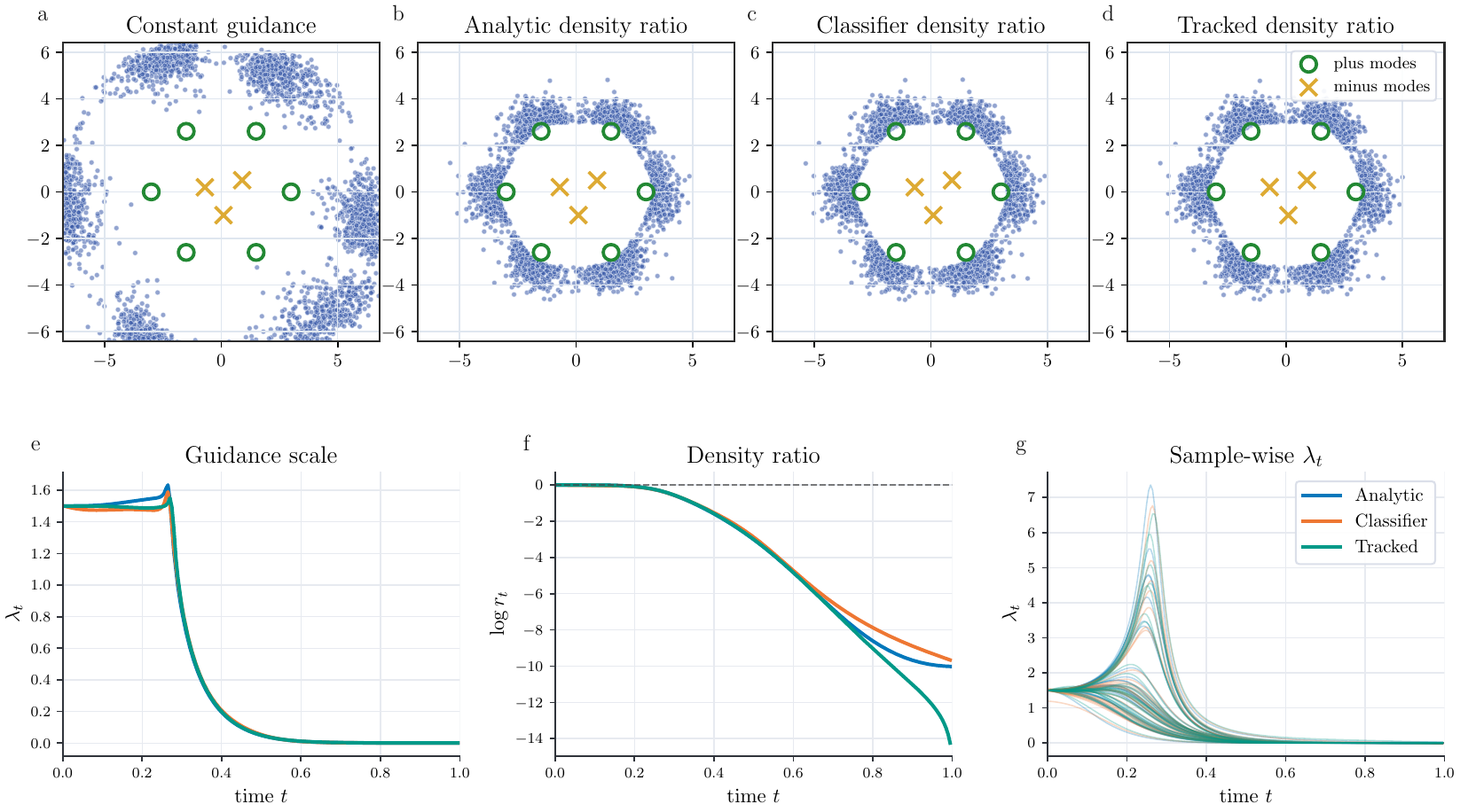}
    \caption{Additional density-ratio estimation ablation in a different setup.}
    \label{fig:density_ratio_ablation_ring}
\end{figure}

\begin{figure}
    \centering
    \includegraphics[width=\textwidth]{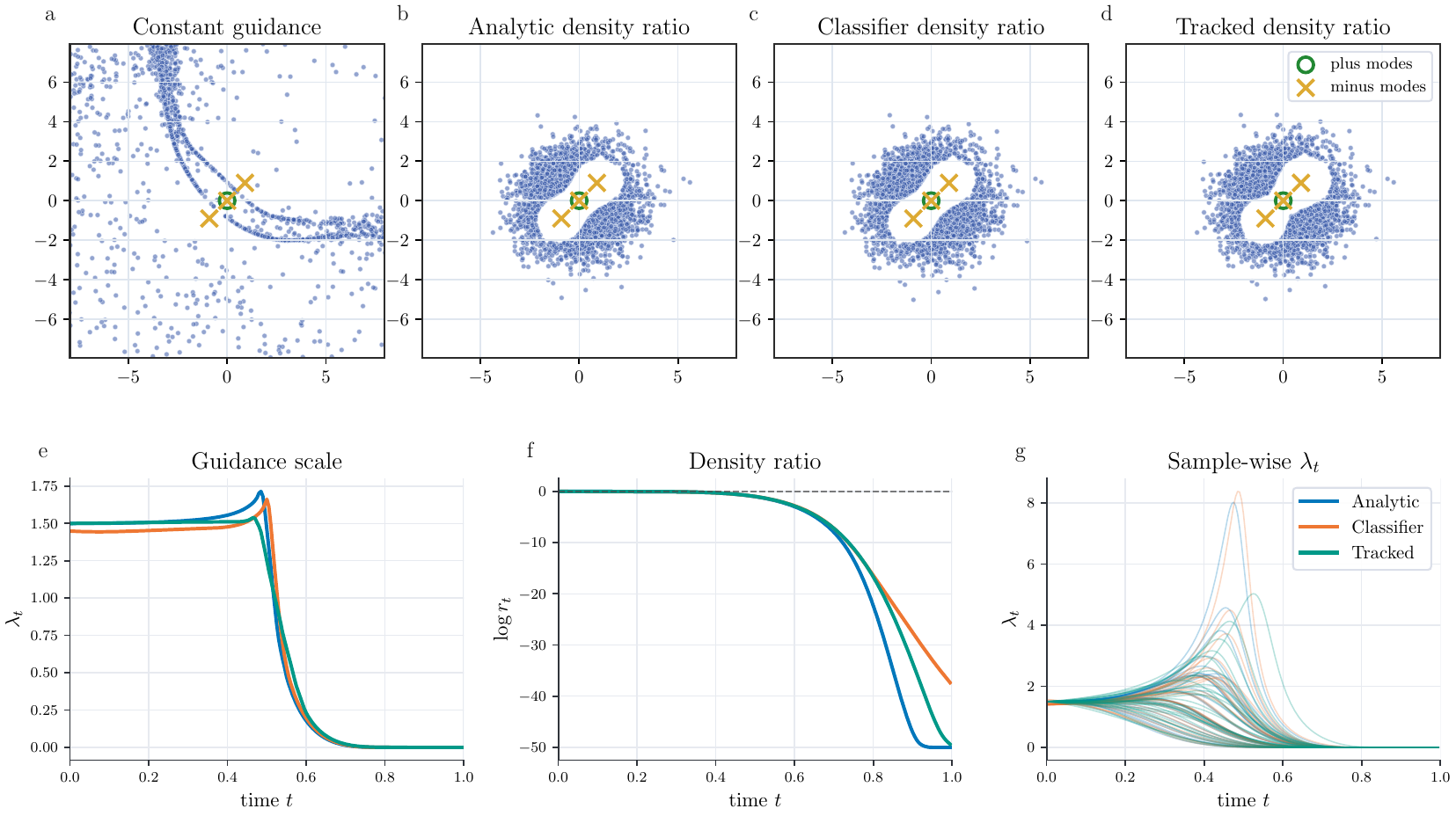}
    \caption{Additional density-ratio estimation ablation in a different setup.}
    \label{fig:density_ratio_ablation_diagonal}
\end{figure}

\subsection{Additional ImageNet-256 Details}
\label{sec:imagenet_cfg_details}

\paragraph{Comparison with related methods.}
Table~\ref{tab:imagenet_cfg_system_comparison} extends the comparison to
several related guidance methods, including ADG~\citep{jin2025angle},
CFG++~\citep{chungcfg++}, and MG~\citep{liao2026momentum}. We report these
results in the appendix because these methods are not derived from the same
signed-target formulation and therefore are not direct baselines for our
theoretical framework. Nevertheless, among methods using the same RF backbone
without additional representation alignment, Signed RF achieves comparable or
lower FID. We report REPA~\citep{yu2025representation} separately because it introduces
vision foundation-model alignment and uses a substantially larger training
budget. Our RF backbone is trained for \(320\) epochs without such alignment,
whereas REPA is trained for \(800\) epochs with VFM alignment. Despite this
difference, Signed RF achieves a comparable FID with substantially fewer
sampling steps: \(1.41\) at \(64\) NFE, compared with \(1.42\) at \(250\) NFE
for REPA.

\begin{table}[t]
    \centering
    \small
    \caption{System-level ImageNet-256 comparison with related guidance baselines.}
    \label{tab:imagenet_cfg_system_comparison}
    \vspace{0.35em}
    \setlength{\tabcolsep}{5pt}
    \renewcommand{\arraystretch}{1.08}
    \begin{tabular}{@{\hspace{8pt}}c@{\hspace{8pt}}c@{\hspace{12pt}}*{3}{>{\centering\arraybackslash}p{0.42in}}@{\hspace{8pt}}}
        \toprule
        NFE & Method & FID \(\downarrow\) & Prec. \(\uparrow\) & Rec. \(\uparrow\) \\
        \midrule
        \multicolumn{5}{@{\hspace{8pt}}l}{\textit{Same RF backbone (320 epochs; no VFM alignment)}} \\
        \multirow{5}{*}{16}
            & CFG & 2.38 & 0.828 & 0.566 \\
            & ADG & 2.32 & 0.827 & 0.570 \\
            & CFG++ & 2.62 & 0.842 & 0.572 \\
            & MG & 1.85 & 0.820 & 0.600 \\
            & Signed RF
            & \textbf{1.82}
            & 0.801
            & 0.615 \\
        \midrule
        \multirow{5}{*}{32}
            & CFG & 1.87 & 0.815 & 0.598 \\
            & ADG & 2.00 & 0.835 & 0.582 \\
            & CFG++ & 1.71 & 0.809 & 0.624 \\
            & MG & 1.71 & 0.810 & 0.620 \\
            & Signed RF
            & \textbf{1.51}
            & 0.800
            & 0.629 \\
        \midrule
        \multirow{5}{*}{64}
            & CFG & 1.73 & 0.821 & 0.601 \\
            & ADG & 1.85 & 0.793 & 0.634 \\
            & CFG++ & 3.39 & 0.838 & 0.573 \\
            & MG & 1.60 & 0.810 & 0.620 \\
            & Signed RF
            & \textbf{1.41}
            & 0.797
            & 0.627 \\
        \midrule
        \multicolumn{5}{@{\hspace{8pt}}l}{\textit{VFM-aligned representation (800 epochs)}} \\
        250 & REPA & 1.42 & 0.800 & 0.650 \\
        \bottomrule
    \end{tabular}
\end{table}

\paragraph{Implementation details.}
We use the pretrained ImageNet \(256\times256\) model released with the
Rectified Flow implementation~\citep{lq2024rectifiedflow} and perform sampling
in the latent space of a
KL-16 autoencoder. At each Euler step, evaluating the model with class label
\(c\) produces the conditional velocity \(v_t^+(\bx_t,c)\), while evaluating
it with the null class token \(c_\varnothing\) produces the unconditional
velocity \(v_t^-(\bx_t)\).

The CFG baseline uses the constant-guidance field
\(
    \hat v_t^{\mathrm{CFG}}(\bx_t,c)
    =
    v_t^+(\bx_t,c)
    +
    \omega
    \bigl(
        v_t^+(\bx_t,c)-v_t^-(\bx_t)
    \bigr).
\)
Signed RF uses the same guidance direction but replaces the constant scale
\(\omega\) with the state-dependent scale
\(\lambda_t^\alpha(\bx_t,c)\).

The direction \(v_t^+-v_t^-\) serves two related purposes. First, it strengthens
class conditioning by moving away from directions shared across the broader
unconditional distribution. Second, because the unconditional branch averages
over many classes and modes, it acts as a smoother reference field.
Extrapolating away from this reference can compensate for the smoothing effect
of neural velocity predictions~\citep{liao2026momentum}. Thus, the
unconditional branch is not treated as an undesirable distribution to be
removed; instead, it provides a reference direction for steering samples
toward better class alignment and perceptual quality.

We estimate the density ratio using a separate timestep- and class-conditioned
ViT classifier~\citep{dosovitskiy2020image} trained in the same latent space.
Following the classifier-based estimator in
Sec.~\ref{sec:signed_rf_practice}, the classifier
\(p_t^\phi(y\mid\bx_t,c)\) predicts the branch label
\(y\in\{-,+\}\). With balanced positive and negative sampling,
\[
    \hat r_t(\bx_t,c)
    =
    \frac{
        p_t^\phi(-\mid\bx_t,c)
    }{
        p_t^\phi(+\mid\bx_t,c)
    },
    \qquad
    \lambda_t^\alpha(\bx_t,c)
    =
    \frac{
        \alpha\,\hat r_t(\bx_t,c)
    }{
        (1+\alpha)-\alpha\,\hat r_t(\bx_t,c)
    }.
\]

The classifier training examples are designed to capture both class
misalignment and poor sample quality. Positive endpoints \(\bx_1^+\) are
ImageNet latents paired with their ground-truth class labels. Negative
endpoints \(\bx_1^-\) include class-aligned but low-quality samples generated
by an earlier model checkpoint without CFG, as well as samples paired with an
incorrect conditioning label. The former teaches the classifier to identify
over-smoothed or low-quality states, while the latter teaches it to detect
class-inconsistent states.

For each positive--negative pair, we sample
\(t\sim\mathrm{Unif}[0,1]\) and
\(\bx_0\sim\mathcal N(\boldsymbol 0,\mI)\), and construct
\[
    \bx_t^+
    =
    (1-t)\bx_0+t\bx_1^+,
    \qquad
    \bx_t^-
    =
    (1-t)\bx_0+t\bx_1^-.
\]
The two branches share the same \((\bx_0,t)\) within each pair, reducing
variation unrelated to the endpoint distributions.

A single classifier is shared across all ImageNet classes and receives the
timestep and class label as conditioning. Positive samples are paired with their
ground-truth labels. For negative samples, the condition is set to the
associated sample class with probability \(0.9\) and to a randomly selected
different class with probability \(0.1\). The classifier has \(12\)
Transformer layers, a hidden dimension of \(768\), \(12\) attention heads, and
AdaLN-style timestep and class conditioning~\citep{peebles2023scalable}, giving it roughly
the capacity of a \texttt{DiT-B} model. We train it from scratch using
AdamW~\citep{loshchilov2017decoupled}, with a learning rate of
\(2\times10^{-4}\), weight decay \(0.01\), global batch size \(1024\), and
EMA rate \(0.999\).

During sampling, we apply the stabilization procedures described in
Sec.~\ref{sec:signed_rf_practice}. Specifically, we clip
\(\log\hat r_t\) to \([-20,20]\), lower-bound the denominator by
\(
    \max\!\left(
        (1+\alpha)-\alpha\hat r_t,\,
        10^{-3}
    \right),
\)
and cap the effective guidance scale according to
\(
    \lambda_t^\alpha
    \leftarrow
    \min\!\left(
        \lambda_t^\alpha,\lambda_{\max}
    \right).
\)
All ImageNet comparisons use \(50{,}000\) class-balanced samples. We use the
same initial Gaussian seeds across methods and NFEs, and report the best-FID
operating point from the corresponding sweep over \(\alpha\) or \(\omega\).

\paragraph{Runtime and computational cost.}
We train the density-ratio classifier for \(250\mathrm{K}\) steps on
\(8\) NVIDIA GH200 GPUs using bf16 mixed precision and
\texttt{torch.compile}. Training takes approximately \(10\) hours. For
comparison, the pretrained XL flow backbone requires approximately \(60\)
hours on \(16\) GPUs. Thus, the ratio classifier adds substantially less
training cost than the generative backbone.

At inference time, the main overhead of Signed RF comes from evaluating the
ratio classifier at each Euler step; the scalar computations required to
obtain \(\hat r_t\) and \(\lambda_t^\alpha\) are negligible in comparison. We
measure this overhead on a single NVIDIA GH200 GPU using bf16 mixed precision,
\texttt{torch.compile}, batch size \(100\), and \(64\) Euler steps. We exclude
checkpoint loading, compilation warmup, VAE decoding, and image writing from
the timing. Under this setup, CFG generates \(13.85\) images/s, corresponding to
\(72.22\) ms/image, with \(4.57\) GB of peak GPU memory. Signed RF generates
\(12.44\) images/s, corresponding to \(80.38\) ms/image, with \(5.05\) GB of
peak GPU memory. This amounts to an \(11.3\%\) increase in sampling time, or
\(8.16\) additional ms/image, together with \(0.48\) GB of additional peak
memory. Overall, the additional training and inference costs of Signed RF are
small relative to those of the underlying generative model.

\paragraph{Stability analysis.}
We first evaluate sensitivity to the training stage of the density-ratio
classifier. We use checkpoints collected after \(100\mathrm{K}\),
\(160\mathrm{K}\), \(200\mathrm{K}\), and \(250\mathrm{K}\) updates. For each
checkpoint and NFE, we sweep \(\alpha\) and report the operating point with the
best FID in Table~\ref{tab:imagenet_classifier_checkpoint_stability}. Performance varies
only slightly across checkpoints. The FID ranges are
\(1.824\)--\(1.845\) at \(16\) NFE,
\(1.536\)--\(1.550\) at \(32\) NFE, and
\(1.436\)--\(1.473\) at \(64\) NFE. IS, precision, and recall exhibit similarly
small variations. These results indicate that Signed RF does not require the
ratio classifier to be stopped at a narrowly selected training checkpoint.

We next study sensitivity to the guidance cap \(\lambda_{\max}\), which limits
the effective guidance when
\((1+\alpha)-\alpha\hat r_t\) becomes small because of ratio-estimation or
numerical error. Using the \(250\mathrm{K}\) classifier and fixing
\(\alpha=0.08\), Table~\ref{tab:imagenet_guidance_cap_stability} shows that moderate
caps produce similar results. At \(32\) NFE,
\(\lambda_{\max}\in\{3,5,10\}\) gives FID
\(1.518\), \(1.510\), and \(1.545\), respectively; at \(64\) NFE, the
corresponding FIDs are \(1.413\), \(1.431\), and \(1.477\). Fig.~\ref{fig:imagenet_guidance_scales} visualizes the effective guidance scale
along trajectories for several classes using shared initial seeds. The
guidance profiles vary across both time and class, confirming that Signed RF
applies nontrivial state-dependent guidance rather than a fixed schedule.
Overall, Signed RF is robust to both the classifier training checkpoint and
the choice of a moderate guidance cap.

\begin{figure}[t]
    \centering
    \includegraphics[width=1\linewidth]{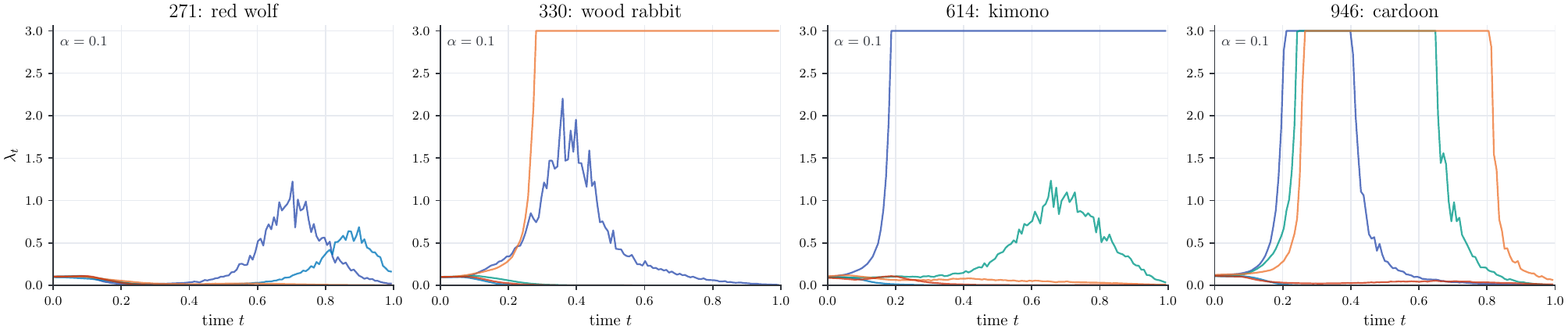}
    \caption{
        Effective guidance scale \(\lambda_t^\alpha\) along trajectories for four ImageNet classes.
        Each color corresponds to the same initial seed shared across panels.
        Despite identical seeds, the guidance profiles differ substantially across classes,
        illustrating the nontrivial and class-dependent behavior of the guidance.
    }
    \label{fig:imagenet_guidance_scales}
\end{figure}

\begin{table}
    \centering
    \small
    \caption{ImageNet classifier-checkpoint stability. All runs use \(\lambda_{\max}\) fixed to 10.}
    \label{tab:imagenet_classifier_checkpoint_stability}
    \vspace{0.5em}
    \setlength{\tabcolsep}{0pt}
    \begin{tabular}{@{\hspace{8pt}}c@{\hspace{9pt}}c@{\hspace{9pt}}c@{\hspace{10pt}}>{\centering\arraybackslash}p{0.39in}@{\hspace{3pt}}>{\centering\arraybackslash}p{0.39in}@{\hspace{3pt}}>{\centering\arraybackslash}p{0.39in}@{\hspace{3pt}}>{\centering\arraybackslash}p{0.39in}@{\hspace{8pt}}}
        \toprule
        NFE & Step & \(\alpha\) & FID \(\downarrow\) & IS \(\uparrow\) & Prec. \(\uparrow\) & Rec. \(\uparrow\) \\
        \midrule
        \multirow{4}{*}{16} & 100K & 0.10 & 1.845 & 299.7 & 0.796 & 0.615 \\
        & 160K & 0.10 & 1.827 & 295.2 & 0.802 & 0.616 \\
        & 200K & 0.10 & 1.824 & 296.8 & 0.801 & 0.615 \\
        & 250K & 0.08 & 1.827 & 293.3 & 0.794 & 0.611 \\
        \midrule
        \multirow{4}{*}{32} & 100K & 0.10 & 1.536 & 286.6 & 0.795 & 0.632 \\
        & 160K & 0.10 & 1.550 & 290.7 & 0.800 & 0.622 \\
        & 200K & 0.10 & 1.549 & 294.3 & 0.800 & 0.626 \\
        & 250K & 0.07 & 1.536 & 290.5 & 0.794 & 0.635 \\
        \midrule
        \multirow{4}{*}{64} & 100K & 0.10 & 1.436 & 278.9 & 0.792 & 0.638 \\
        & 160K & 0.10 & 1.473 & 287.3 & 0.799 & 0.627 \\
        & 200K & 0.10 & 1.471 & 291.4 & 0.799 & 0.627 \\
        & 250K & 0.07 & 1.469 & 289.1 & 0.792 & 0.633 \\
        \bottomrule
    \end{tabular}
\end{table}

\begin{table}[H]
    \centering
    \small
    \caption{ImageNet guidance-cap stability. We use the 250K classifier, fix \(\alpha\) to 0.08, and vary \(\lambda_{\max}\).}
    \label{tab:imagenet_guidance_cap_stability}
    \vspace{0.5em}
    \setlength{\tabcolsep}{0pt}
    \begin{tabular}{@{\hspace{8pt}}c@{\hspace{9pt}}c@{\hspace{10pt}}>{\centering\arraybackslash}p{0.39in}@{\hspace{3pt}}>{\centering\arraybackslash}p{0.39in}@{\hspace{3pt}}>{\centering\arraybackslash}p{0.39in}@{\hspace{3pt}}>{\centering\arraybackslash}p{0.39in}@{\hspace{8pt}}}
        \toprule
        NFE & \(\lambda_{\max}\) & FID \(\downarrow\) & IS \(\uparrow\) & Prec. \(\uparrow\) & Rec. \(\uparrow\) \\
        \midrule
        \multirow{3}{*}{16} & 3  & 1.918 & 275.8 & 0.793 & 0.615 \\
        & 5  & 1.845 & 289.0 & 0.797 & 0.612 \\
        & 10 & 1.827 & 293.3 & 0.794 & 0.611 \\
        \midrule
        \multirow{3}{*}{32} & 3  & 1.518 & 278.4 & 0.798 & 0.630 \\
        & 5  & 1.510 & 288.6 & 0.800 & 0.629 \\
        & 10 & 1.545 & 293.6 & 0.796 & 0.625 \\
        \midrule
        \multirow{3}{*}{64} & 3  & 1.413 & 278.2 & 0.797 & 0.627 \\
        & 5  & 1.431 & 286.6 & 0.797 & 0.627 \\
        & 10 & 1.477 & 292.0 & 0.795 & 0.632 \\
        \bottomrule
    \end{tabular}
\end{table}

\subsection{Additional Analytic Negative-Flow Details}
\label{sec:analytic_negative_flow_details}

\paragraph{Analytic flow under a Gaussian source.}
Let \(\pi_1\) be the empirical distribution supported on
\(\{\bx_1^{(i)}\}_{i=1}^N\subset\mathbb R^D\):
\vspace{-0.8em}
\[
    \pi_1(\bx)
    =
    \frac{1}{N}
    \sum_{i=1}^{N}
    \delta_{\bx_1^{(i)}}(\bx).
\]
Consider the RF interpolation
\(
    \bX_t
    =
    t\bX_1+(1-t)\bX_0,
    \qquad
    \bX_0\sim\mathcal N(\boldsymbol 0,\mI).
\)
The corresponding velocity field is
\(
    v_t(\bx)
    =
    \mathbb E\!\left[
        \bX_1-\bX_0
        \,\middle|\,
        \bX_t=\bx
    \right].
\) For \(0\leq t<1\), conditioned on \(\bX_1=\bx_1^{(i)}\), the intermediate
state follows
\[
    \bX_t
    \sim
    \mathcal N\!\left(
        t\bx_1^{(i)},
        (1-t)^2\mI
    \right).
\]
Averaging over the empirical data points therefore gives the time-\(t\)
marginal
\begin{equation}
    \pi_t(\bx)
    =
    \frac{1}{N}
    \sum_{i=1}^{N}
    \mathcal N\!\left(
        \bx;\,
        t\bx_1^{(i)},
        (1-t)^2\mI
    \right).
    \label{eq:analytic_flow_marginal}
\end{equation}

Given \(\bX_t=\bx\), the posterior probability that the endpoint is
\(\bx_1^{(i)}\) is
\[
    w_i(\bx,t)
    \coloneqq
    \mathbb P\!\left(
        \bX_1=\bx_1^{(i)}
        \,\middle|\,
        \bX_t=\bx
    \right),
\]
where
\begin{equation}
    w_i(\bx,t)
    =
    \frac{
        \exp\!\left(
            -\|\bx-t\bx_1^{(i)}\|^2
            \big/
            \bigl(2(1-t)^2\bigr)
        \right)
    }{
        \displaystyle
        \sum_{j=1}^{N}
        \exp\!\left(
            -\|\bx-t\bx_1^{(j)}\|^2
            \big/
            \bigl(2(1-t)^2\bigr)
        \right)
    }.
    \label{eq:analytic_flow_weights}
\end{equation}

Using
\(
    \bX_0=(\bX_t-t\bX_1)/(1-t)
\),
we obtain
\begin{align}
    v_t(\bx)
    &=
    \mathbb E\!\left[
        \bX_1-\bX_0
        \,\middle|\,
        \bX_t=\bx
    \right]
    \notag
    \\
    &=
    \mathbb E\!\left[
        \frac{\bX_1-\bx}{1-t}
        \,\middle|\,
        \bX_t=\bx
    \right]
    \notag
    \\
    &=
    \sum_{i=1}^{N}
    w_i(\bx,t)
    \frac{\bx_1^{(i)}-\bx}{1-t}.
    \label{eq:analytic_flow_velocity}
\end{align}
Thus, the exact velocity is obtained by moving toward the posterior-weighted
average of the empirical endpoints.

For numerical evaluation, the corresponding log marginal can be written as
\begin{equation}
    \log\pi_t(\bx)
    =
    \log
    \sum_{i=1}^{N}
    \exp\!\left(
        -\frac{
            \|\bx-t\bx_1^{(i)}\|^2
        }{
            2(1-t)^2
        }
    \right)
    -
    \log N
    -
    \frac{D}{2}\log(2\pi)
    -
    D\log(1-t).
    \label{eq:analytic_flow_log_density}
\end{equation}
These density, weight, velocity, and log-density formulas apply for \(t<1\);
at \(t=1\), the terminal law is the empirical measure \(\pi_1\).

\paragraph{PyTorch implementation.}
The implementation below evaluates the analytic velocity for a reference set
of data latents with shape \([N,C,H,W]\) and a batch of intermediate states
\(\bx_t\) with shape \([B,C,H,W]\). After flattening the latents, it computes
the \(B\times N\) matrix of squared distances
\(\|\bx_t-t\bx_1^{(i)}\|^2\). A softmax over the reference points gives the
posterior weights in Eq.~\ref{eq:analytic_flow_weights}, from
which the posterior-weighted endpoint mean and the velocity are obtained.
When requested, the same unnormalized log weights are reused to evaluate
\(\log\pi_t(\bx_t)\) through a log-sum-exp operation.

\begin{paperpythonlisting}{PyTorch implementation of the Gaussian analytic flow}
class GaussianAnalyticFlow:
    def __init__(self, data):
        # data: [N, C, H, W], e.g., ImageNet autoencoder latents.
        self.data_shape = tuple(data.shape[1:])
        self.data = data.detach().flatten(1).float()
        self.data_sq_norm = self.data.square().sum(dim=1)

    @torch.no_grad()
    def velocity_and_logp(self, x_t, t, return_logp=False):
        # x_t: [B, C, H, W]
        # t: scalar tensor or tensor of shape [B]
        batch_size = x_t.shape[0]
        if t.ndim == 0:
            t = t.expand(batch_size)
        t = t.to(device=x_t.device, dtype=torch.float32)

        x_shape = x_t.shape
        x_flat = x_t.flatten(1).float()
        _, latent_dim = x_flat.shape

        data = self.data.to(
            device=x_t.device,
            dtype=x_flat.dtype,
        )
        data_sq_norm = self.data_sq_norm.to(
            device=x_t.device,
            dtype=x_flat.dtype,
        )
        num_data = data.shape[0]

        # Compute ||x_t - t x_1^{(i)}||^2 for every batch--data pair.
        x_sq_norm = x_flat.square().sum(dim=1)
        x_data_inner_product = x_flat @ data.T
        sq_dist = (
            x_sq_norm[:, None]
            - 2.0 * t[:, None] * x_data_inner_product
            + t.square()[:, None] * data_sq_norm[None, :]
        )

        one_minus_t = (1.0 - t).clamp_min(1e-12)
        log_weights = (
            -0.5
            * sq_dist
            / one_minus_t.square()[:, None]
        )
        weights = torch.softmax(log_weights, dim=1)

        # Posterior mean E[X_1 | X_t = x_t].
        endpoint_mean = weights @ data
        velocity_flat = (
            endpoint_mean - x_flat
        ) / one_minus_t[:, None]
        velocity = velocity_flat.reshape(x_shape)

        if not return_logp:
            return velocity, None

        log_mixture = (
            torch.logsumexp(log_weights.float(), dim=1)
            - math.log(float(num_data))
        )
        log_normalizer = (
            -0.5 * latent_dim * math.log(2.0 * math.pi)
            - latent_dim * torch.log(one_minus_t.float())
        )
        logp = log_mixture + log_normalizer

        return velocity, logp
\end{paperpythonlisting}

\paragraph{ImageNet anti-memorization sampler.}
For the anti-memorization experiments in
Sec.~\ref{sec:anti_memorization}, we construct a separate analytic
negative branch for each ImageNet class. Given the training latents of class
\(c\),
\[
    \mathcal D_c
    =
    \{\bx_1^{(i)}\}_{i=1}^{N_c}
    \subset
    \mathbb R^D,
    \qquad
    \pi_{1,c}^-(\bx)
    =
    \frac{1}{N_c}
    \sum_{i=1}^{N_c}
    \delta_{\bx_1^{(i)}}(\bx),
\]
Eqs.~\ref{eq:analytic_flow_marginal},
\ref{eq:analytic_flow_velocity}, and
\ref{eq:analytic_flow_log_density} provide the corresponding
marginal, negative velocity \(v_t^-\), and log-density
\(\log\pi_{t,c}^-\) in closed form. We augment each class-specific reference
set \(\mathcal D_c\) with horizontally flipped training latents.

The positive velocity \(v_t^+\) is provided by a pretrained class-conditional
ImageNet flow model using CFG with \(\omega=1.0\). Because the positive-branch
density is not available in closed form, we use the online-tracked ratio
described in Sec.~\ref{sec:signed_rf_practice}. For each trajectory, we
maintain
\(
    u_t
    \approx
    \log r_t(\bx_t)
    =
    \log
    \frac{\pi_{t,c}^-(\bx_t)}
         {\pi_{t,c}^+(\bx_t)},
\)
with \(u_0=0\), since both branches share the same Gaussian source.

At each Euler step, we evaluate the pretrained positive velocity \(v_t^+\)
and the analytic negative velocity \(v_t^-\). We then recover the branch
scores using the Gaussian-source RF identity
\(
    s_t^\pm(\bx)
    =
    \frac{
        t\,v_t^\pm(\bx)-\bx
    }{
        1-t
    }
\)~\citep{liu2025let},
and update \(u_t\) according to
Eq.~\ref{eq:online_ratio_tracking_restate}. The only required divergence term is
\(
    \nabla\!\cdot
    \bigl(
        v_t^+-v_t^-
    \bigr),
\)
which we estimate using the Hutchinson trace
estimator~\citep{hutchinson1989stochastic}. We apply the same ratio clipping,
denominator lower bound, and guidance cap used in the ImageNet-256 experiments
above.

\paragraph{Evaluation protocol.}
For the targeted memorization stress test, we first generate a pool of
\(50\mathrm{K}\) base-model samples for each selected class and compute the
nearest-neighbor distance from each sample to the corresponding class-specific
training set. We rank samples using the \(L_2\) distance between SSCD
embeddings, which is more sensitive to copied visual details than pixel-space
or latent-space distances in our ImageNet experiments.

For each class, we retain the \(600\) highest-risk initial seeds, defined as
those producing the smallest SSCD \(L_2\) nearest-neighbor distances under the
base model. Starting from exactly these same initial noise samples, we then
rerun the base model, the SPELL baselines, and Data Repulsive Flow. This
shared-seed protocol ensures that differences in memorization are attributable
to the sampling method rather than to the selected initial states.

We report the \(5\)th, \(10\)th, and \(25\)th percentiles, together with the
mean, of the SSCD \(L_2\) nearest-neighbor distance. Larger values indicate
that the generated samples are farther from the training set and therefore
exhibit less training-data similarity. For the standard generation-quality
evaluation, we instead generate \(50\mathrm{K}\) class-balanced samples from
random initial noise, while sharing the same initial seeds across all compared
methods.

Table~\ref{tab:imagenet_memorization_per_class} extends the
class-\(248\) results from the main paper to five ImageNet classes and several
values of \(\alpha\). Across all evaluated classes, increasing \(\alpha\)
consistently raises the lower-tail nearest-neighbor distances. This shows that
Data Repulsive Flow provides a smooth and interpretable control over the
strength of training-data repulsion, rather than exhibiting an abrupt change
at a particular setting.

We additionally report nearest-neighbor distances in the autoencoder latent
space. These distances are not used to select the high-risk seeds and therefore
serve as an independent evaluation metric. Their consistent increase supports
the conclusion that the observed effect is not specific to the SSCD
\(L_2\) metric used for seed selection and primary reporting.

\begin{table}[t]
    \centering
    \renewcommand{\arraystretch}{1.08}
    \caption{
        Additional per-class \(\alpha\) sweeps for ImageNet anti-memorization.
        For each class, we report nearest-neighbor SSCD \(L_2\) and latent \(L_2\) distances on the same \(600\) high-risk initial seeds used by the base model and Data Repulsive Flow.
        Larger values indicate lower similarity to the nearest training example.
        The seeds are selected by base-model SSCD \(L_2\); latent \(L_2\) is reported as an additional metric.
        For each metric we report the lower-tail \(P_{05}\), \(P_{10}\), and \(P_{25}\) quantiles and the mean.
    }
    \label{tab:imagenet_memorization_per_class}
    \vspace{0.5em}
    \small
    \setlength{\tabcolsep}{2.5pt}
    \begin{tabular}{@{}c@{\hspace{10pt}}c@{\hspace{11pt}}c@{\hspace{16pt}}cccc@{\hspace{18pt}}cccc@{}}
        \toprule
        \multirow{2}{*}{Class} & \multirow{2}{*}{Method} & \multirow{2}{*}{\(\alpha\)}
        & \multicolumn{4}{c}{SSCD \(L_2\) \(\uparrow\)}
        & \multicolumn{4}{c}{Latent \(L_2\) \(\uparrow\)} \\
        \cmidrule(lr){4-7}\cmidrule(l){8-11}
        & & & \(P_{05}\) & \(P_{10}\) & \(P_{25}\) & Mean & \(P_{05}\) & \(P_{10}\) & \(P_{25}\) & Mean \\
        \midrule
        \multirow{5}{*}{\shortstack[c]{151\\(Chihuahua)}} & Base & -- & 0.921 & 0.974 & 1.017 & 1.033 & 47.27 & 51.66 & 58.39 & 62.98 \\
        & Ours & 0.1 & 0.928 & 0.976 & 1.020 & 1.040 & 48.38 & 51.96 & 58.95 & 63.22 \\
        & Ours & 1.0 & 0.944 & 0.991 & 1.033 & 1.068 & 48.83 & 53.83 & 60.35 & 64.21 \\
        & Ours & 5.0 & 0.983 & 1.013 & 1.063 & 1.107 & 54.28 & 57.04 & 61.66 & 65.23 \\
        & Ours & 7.5 & \textbf{0.993} & \textbf{1.023} & \textbf{1.073} & \textbf{1.119} & \textbf{54.50} & \textbf{57.72} & \textbf{61.93} & \textbf{65.59} \\
        \midrule
        \multirow{5}{*}{\shortstack[c]{248\\(Eskimo dog,\\ husky)}} & Base & -- & 0.921 & 0.944 & 0.982 & 0.993 & 54.88 & 56.99 & 60.59 & 63.81 \\
        & Ours & 1.0 & 0.939 & 0.962 & 0.992 & 1.024 & 56.31 & 58.15 & 61.29 & 64.28 \\
        & Ours & 3.0 & 0.956 & 0.975 & 1.005 & 1.045 & 57.71 & 58.78 & 61.62 & 64.63 \\
        & Ours & 5.0 & 0.958 & 0.979 & 1.012 & 1.056 & 57.69 & 59.28 & 61.87 & 64.99 \\
        & Ours & 7.5 & \textbf{0.964} & \textbf{0.987} & \textbf{1.020} & \textbf{1.066} & \textbf{57.99} & \textbf{59.91} & \textbf{62.36} & \textbf{65.36} \\
        \midrule
        \multirow{5}{*}{\shortstack[c]{555\\(fire engine,\\ fire truck)}} & Base & -- & 0.967 & 0.978 & 0.998 & 1.010 & 54.09 & 55.58 & 57.54 & 59.69 \\
        & Ours & 0.1 & 0.968 & 0.980 & 0.998 & 1.013 & 54.27 & 55.62 & 57.54 & 59.72 \\
        & Ours & 1.0 & 0.969 & 0.983 & 1.002 & 1.031 & 54.42 & 55.97 & 57.67 & 59.80 \\
        & Ours & 5.0 & 0.978 & 0.997 & 1.026 & 1.070 & 55.19 & 56.24 & 58.27 & 60.41 \\
        & Ours & 7.5 & \textbf{0.983} & \textbf{1.004} & \textbf{1.039} & \textbf{1.081} & \textbf{55.49} & \textbf{56.55} & \textbf{58.54} & \textbf{60.72} \\
        \midrule
        \multirow{5}{*}{\shortstack[c]{805\\(soccer ball)}} & Base & -- & 0.854 & 0.873 & 0.904 & 0.929 & 48.96 & 50.17 & 52.80 & 59.07 \\
        & Ours & 0.1 & 0.852 & 0.875 & 0.911 & 0.936 & 49.17 & 50.50 & 53.05 & 59.40 \\
        & Ours & 1.0 & 0.865 & 0.885 & 0.922 & 0.977 & 50.15 & 51.64 & 54.47 & 60.62 \\
        & Ours & 5.0 & 0.890 & 0.919 & 0.968 & 1.050 & 51.45 & 53.91 & 57.39 & 62.82 \\
        & Ours & 7.5 & \textbf{0.902} & \textbf{0.939} & \textbf{0.993} & \textbf{1.073} & \textbf{52.67} & \textbf{54.44} & \textbf{58.22} & \textbf{63.78} \\
        \midrule
        \multirow{5}{*}{\shortstack[c]{975\\(lakeside,\\ lakeshore)}} & Base & -- & 0.949 & 0.992 & 1.042 & 1.061 & 40.69 & 43.14 & 46.66 & 51.38 \\
        & Ours & 0.1 & 0.964 & 1.004 & 1.051 & 1.079 & 42.06 & 44.01 & 47.31 & 51.85 \\
        & Ours & 1.0 & 0.981 & 1.036 & 1.079 & 1.136 & 43.03 & 45.05 & 48.66 & 52.98 \\
        & Ours & 5.0 & 1.025 & 1.065 & 1.149 & 1.186 & 43.99 & 45.88 & 49.70 & 53.33 \\
        & Ours & 7.5 & \textbf{1.039} & \textbf{1.089} & \textbf{1.164} & \textbf{1.195} & \textbf{44.23} & \textbf{45.94} & \textbf{50.05} & \textbf{53.64} \\
        \bottomrule
    \end{tabular}
\end{table}

\paragraph{Runtime and memory cost.}
The analytic negative branch requires no additional training. For a batch of
size \(B\), a reference set of size \(N\), and latent dimension \(D\), its
inference cost is dominated by computing the state--reference Gram matrix and
the responsibility-weighted endpoint average. Both operations scale as
\(\mathcal O(BND)\). The memory cost consists primarily of the reference
matrix, which requires \(\mathcal O(ND)\) storage, and the temporary
\(B\times N\) matrices of logits and posterior responsibilities.

In our class-conditional ImageNet experiments, each class contains
approximately \(1.3\mathrm{K}\) reference latents, or
\(2.6\mathrm{K}\) after horizontal-flip augmentation. At this scale, the
analytic branch is small relative to the pretrained flow backbone in both
runtime and memory.

Fig.~\ref{fig:analytic_flow_runtime_memory} benchmarks the analytic branch on a
single NVIDIA GH200 GPU with the ImageNet latent dimension \(D=4096\).
Velocity-only, log-density-only, and joint evaluation have nearly identical
runtime because all three reuse the same \(B\times N\) distance logits. At the
class-level reference-set sizes used in our experiments, each evaluation takes
only a few milliseconds and requires substantially less than \(0.1\) GB of
additional memory.

Both runtime and reference storage grow approximately linearly with \(N\).
At \(N=2^{18}\), evaluation remains below \(8\) ms. At
\(N=2^{22}\), it takes on the order of \(10^2\) ms, while storing the
reference set in fp32 requires \(64\) GB. These measurements confirm that the
class-specific analytic branch is not a computational bottleneck in our ImageNet
anti-memorization experiments. For substantially larger protected datasets,
the reference set can be processed in chunks while incrementally accumulating
the log-sum-exp normalizer and responsibility-weighted endpoint average. This
keeps memory bounded at the cost of additional passes over the reference set.

\begin{figure}[t]
    \centering
    \includegraphics[width=1\linewidth]{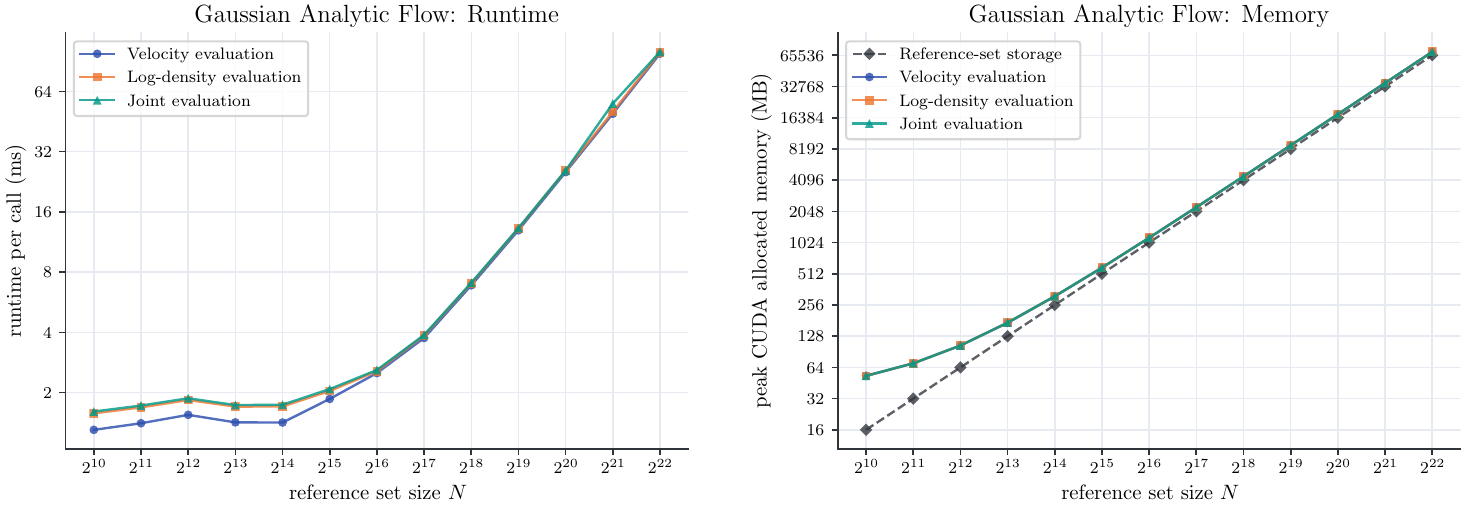}
    \caption{
        Runtime and peak CUDA memory of the vectorized Gaussian analytic flow on an NVIDIA GH200 GPU.
        The benchmark sweeps the reference-set size \(N\) at fixed ImageNet latent dimension \(D=4096\).
        Joint velocity and log-density evaluation is nearly as cheap as either quantity alone because the same squared-distance logits are reused.
    }
    \label{fig:analytic_flow_runtime_memory}
\end{figure}

\subsection{Additional Concept Suppression Details}
\label{sec:concept_suppression_prompts}

For the main-text concept-suppression examples, the positive prompt defines the desired visual content and the negative prompt defines the concept or attribute to suppress. We use the following prompt pairs verbatim.

\begin{enumerate}[leftmargin=1.6em,itemsep=0.35em,topsep=0.25em]
    \item \textbf{Positive:} A high-aesthetic anime portrait, upper body, night street background, expressive eyes, polished shading, clean linework, masterpiece, best quality\\
    \textbf{Negative:} lowres, blurry, bad anatomy, bad hands, missing fingers, extra digits, jpeg artifacts, cropped, low quality, watermark, signature

    \item \textbf{Positive:} A clean contemporary house, strong geometric lines, large glass panels, minimalist architecture photography, crisp daylight, centered framing, unobstructed surroundings, pure architectural focus\\
    \textbf{Negative:} trees, bushes, leaves, greenery, vines, overgrown plants, grass, lawn, meadow, turf, wildflowers, landscaped garden

    \item \textbf{Positive:} A single large paper crane flying low over a quiet lake at dawn, close-up, centered composition, the crane fills much of the frame, poetic, minimalist, soft mist\\
    \textbf{Negative:} origami, folded paper, craft, crease

    \item \textbf{Positive:} A solitary young woman sitting by a large train window during rain, soft city lights outside, reflective atmosphere, cinematic portrait\\
    \textbf{Negative:} sad, crying, depressed, grief, despair, tears, misery

    \item \textbf{Positive:} A cheerful cartoon mouse with perfectly round black ears, red shorts, yellow shoes, white gloves, simple face, vintage mascot design, standing confidently against a clean pastel background\\
    \textbf{Negative:} Mickey Mouse, Disney character

    \item \textbf{Positive:} A pale angel standing silently in a winter garden, marble-like skin, white wings dusted with snow, stillness, ethereal realism\\
    \textbf{Negative:} stone statue, sculpture pedestal, carved marble, museum artifact

    \item \textbf{Positive:} A world-class scientist presenting a breakthrough discovery in a sleek research lab, confident, inspiring, magazine-cover composition\\
    \textbf{Negative:} old white man, male stereotype, lab coat cliche, Einstein-like appearance

    \item \textbf{Positive:} A visionary CEO standing in a glass office overlooking a futuristic city, calm authority, elegant natural light\\
    \textbf{Negative:} middle-aged white man in dark suit, corporate cliche, Wall Street stereotype

    \item \textbf{Positive:} A dynamic cinematic shot of a lean teenage superhero crouching high above a city, wearing an iconic red-and-blue skin-tight suit defined by black webbing lines, a bold spider emblem on the chest, and oversized white pointed eye lenses on a full-face mask, athletic, agile, youthful, with a strong sense of motion and classic urban superhero energy\\
    \textbf{Negative:} Spider-Man, Marvel, Peter Parker, superhero franchise

    \item \textbf{Positive:} Sweeping tulip fields under soft spring sunlight, wide-angle landscape view, richly layered flowers, fresh air, elegant color harmony, cinematic countryside photography\\
    \textbf{Negative:} red, crimson, scarlet, ruby tones
\end{enumerate}

\end{document}